\documentclass{article}

\usepackage{microtype}
\usepackage{graphicx}
\graphicspath{{figures/} {figures/GMM-poi/}{figures/GMM-poi90/}{figures/GMM-demo/} {figures/PCFG/}}
\usepackage{subfigure}
\usepackage{booktabs} %
\usepackage{listings}
\usepackage{dblfloatfix}
\usepackage{float}
\usepackage{amsfonts}
\usepackage{amssymb}
\usepackage{bbm}
\usepackage{mathrsfs}
\usepackage{enumitem}
\usepackage{hyperref}
\usepackage{todonotes}

\usepackage{amsmath}
\usepackage{amsthm}
\usepackage{amssymb}
\usepackage{thm-restate}
\usepackage[color]{changebar}

\usepackage{graphicx}
\usepackage{graphbox}
\usepackage{wrapfig}

\newcommand*{\prog}{\mathit{prog}}

\usepackage{array}
\makeatletter
\newcommand{\thickhline}{%
	\noalign {\ifnum 0=`}\fi \hrule height 1pt
	\futurelet \reserved@a \@xhline
}
\newcolumntype{"}{@{\hskip\tabcolsep\vrule width 1pt\hskip\tabcolsep}}
\makeatother

\newcommand{\iden}{\ensuremath{\mathbb{I}}}

\newcommand{\pto}{\overset{p}{\to}}
\newcommand{\dto}{\overset{d}{\to}}

\newtheorem{assumption}{Assumption}

\usepackage{listings}
\usepackage{color}
\usepackage{xcolor}
\definecolor{blue}{rgb}{0,0.3,0.7}
\definecolor{red}{rgb}{0.60,0.0,0.0}
\definecolor{purple}{rgb}{0.5,0,0.7}
\definecolor{cyan}{rgb}{0.0,0.6,0.5}
\definecolor{gray}{rgb}{0.4,0.4,0.4}

\lstdefinelanguage{scheme}
{sensitive, %
 alsoletter={:,-,+,*,?,/,!,>,<}, %
 morecomment=[l]{;}, %
}[comments]

\lstdefinelanguage{anglican}%
{%
 morekeywords=[1]{},
 morekeywords=[2]{%
   def, def-, defn, defn-, defmacro, defmulti, defmethod, %
   defstruct, defonce, declare, definline, definterface, %
   defprotocol, defrecord, defstruct, deftype, defproject, ns, %
 }, %
 morekeywords=[3]{->, ->>, .., amap, and, areduce, as->, assert, binding, %
   bound-fn, case, comment, cond, cond->, cond->>, condp, declare, definline, %
   definterface, defmacro, defmethod, defmulti, defn, defn-, defonce, %
   defprotocol, defrecord, defstruct, deftype, delay, doseq, dosync, dotimes, %
   doto, extend-protocol, extend-type, fn, for, future, gen-class, %
   gen-interface, if, if-let, if-not, if-some, import, io!, lazy-cat, lazy-seq, let, %
   letfn, locking, loop, memfn, ns, or, proxy, proxy-super, pvalues, %
   recur, refer-clojure, reify, some->, some->>, sync, time, when, when-first, %
   when-let, when-not, when-some, while, with-bindings, with-in-str, %
   with-loading-context, with-local-vars, with-open, with-out-str, %
   with-precision, with-redefs, else}, %
  morekeywords=[4]{*, *', +, +', -, -', ->ArrayChunk, ->Vec, ->VecNode, %
    ->VecSeq, -cache-protocol-fn, -reset-methods, /, <, <=, =, ==, >, >=, %
    accessor, aclone, add-classpath, add-watch, agent, agent-error, %
    agent-errors, aget, alength, alias, all-ns, alter, alter-meta!, %
    alter-var-root, ancestors, apply, array-map, aset, aset-boolean, aset-byte, %
    aset-char, aset-double, aset-float, aset-int, aset-long, aset-short, assoc, %
    assoc!, assoc-in, associative?, atom, await, await-for, await1, bases, bean, %
    bigdec, bigint, biginteger, bit-and, bit-and-not, bit-clear, bit-flip, %
    bit-not, bit-or, bit-set, bit-shift-left, bit-shift-right, bit-test, %
    bit-xor, boolean, boolean-array, booleans, bound-fn*, bound?, butlast, byte, %
    byte-array, bytes, cast, char, char-array, char?, chars, chunk, %
    chunk-append, chunk-buffer, chunk-cons, chunk-first, chunk-next, chunk-rest, %
    chunked-seq?, class, class?, clear-agent-errors, clojure-version, coll?, %
    commute, comp, comparator, compare, compare-and-set!, compile, complement, %
    concat, conj, conj!, cons, constantly, construct-proxy, contains?, count, %
    counted?, create-ns, create-struct, cycle, dec, dec', decimal?, delay?, %
    deliver, denominator, deref, derive, descendants, destructure, disj, disj!, %
    dissoc, dissoc!, distinct, distinct?, doall, dorun, double, double-array, %
    doubles, drop, drop-last, drop-while, empty, empty?, ensure, %
    enumeration-seq, error-handler, error-mode, eval, even?, every-pred, every?, %
    ex-data, ex-info, extend, extenders, extends?, false?, ffirst, file-seq, %
    filter, filter-ns-publics, filterv, find, find-keyword, find-ns, %
    find-protocol-impl, find-protocol-method, find-var, first, flatten, float, %
    float-array, float?, floats, flush, fn?, fnext, fnil, force, format, %
    frequencies, future-call, future-cancel, future-cancelled?, future-done?, %
    future?, gensym, get, get-in, get-method, get-proxy-class, %
    get-thread-bindings, get-validator, group-by, hash, hash-combine, hash-map, %
    hash-ordered-coll, hash-set, hash-unordered-coll, identical?, identity, %
    ifn?, in-ns, inc, inc', init-proxy, instance?, int, int-array, integer?, %
    interleave, intern, interpose, into, into-array, ints, isa?, iterate, %
    iterator-seq, juxt, keep, keep-indexed, key, keys, keyword, keyword?, last, %
    line-seq, list, list*, list?, load, load-file, load-reader, load-string, %
    loaded-libs, long, long-array, longs, macroexpand, macroexpand-1, %
    make-array, make-hierarchy, map, map-indexed, map?, mapcat, mapv, max, %
    max-key, memoize, merge, merge-with, meta, method-sig, methods, min, %
    min-key, mix-collection-hash, mod, munge, name, namespace, namespace-munge, %
    neg?, newline, next, nfirst, nil?, nnext, not, not-any?, not-empty, %
    not-every?, not=, ns-aliases, ns-functions, ns-imports, ns-interns, %
    ns-macros, ns-map, ns-name, ns-publics, ns-refers, ns-resolve, ns-unalias, %
    ns-unmap, nth, nthnext, nthrest, num, number?, numerator, object-array, %
    odd?, parents, partial, partition, partition-all, partition-by, pcalls, %
    peek, persistent!, pmap, pop, pop!, pop-thread-bindings, pos?, pr, pr-str, %
    prefer-method, prefers, print, print-ctor, print-simple, print-str, printf, %
    println, println-str, prn, prn-str, promise, proxy-call-with-super, %
    proxy-mappings, proxy-name, push-thread-bindings, quot, rand, rand-int, %
    rand-nth, range, ratio?, rational?, rationalize, re-find, re-groups, %
    re-matcher, re-matches, re-pattern, re-seq, read, read-line, read-string, %
    realized?, record?, reduce, reduce-kv, reduced, reduced?, reductions, ref, %
    ref-history-count, ref-max-history, ref-min-history, ref-set, refer, %
    release-pending-sends, rem, remove, remove-all-methods, remove-method, %
    remove-ns, remove-watch, repeat, repeatedly, replace, replicate, require, %
    reset!, reset-meta!, resolve, rest, restart-agent, resultset-seq, reverse, %
    reversible?, rseq, rsubseq, satisfies?, second, select-keys, send, send-off, %
    send-via, seq, seq?, seque, sequence, sequential?, set, %
    set-agent-send-executor!, set-agent-send-off-executor!, set-error-handler!, %
    set-error-mode!, set-validator!, set?, short, short-array, shorts, shuffle, %
    shutdown-agents, slurp, some, some-fn, some?, sort, sort-by, sorted-map, %
    sorted-map-by, sorted-set, sorted-set-by, sorted?, special-symbol?, spit, %
    split-at, split-with, str, string?, struct, struct-map, subs, subseq, %
    subvec, supers, swap!, symbol, symbol?, take, take-last, take-nth, %
    take-while, test, the-ns, thread-bound?, to-array, to-array-2d, trampoline, %
    transient, tree-seq, true?, type, unchecked-add, unchecked-add-int, %
    unchecked-byte, unchecked-char, unchecked-dec, unchecked-dec-int, %
    unchecked-divide-int, unchecked-double, unchecked-float, unchecked-inc, %
    unchecked-inc-int, unchecked-int, unchecked-long, unchecked-multiply, %
    unchecked-multiply-int, unchecked-negate, unchecked-negate-int, %
    unchecked-remainder-int, unchecked-short, unchecked-subtract, %
    unchecked-subtract-int, underive, unsigned-bit-shift-right, update-in, %
    update-proxy, use, val, vals, var-get, var-set, var?, vary-meta, vec, %
    vector, vector-of, vector?, with-bindings*, with-meta, with-redefs-fn, %
    xml-seq, zero?, zipmap}, %
  morekeywords=[5]{def-cps-fn, defanglican, defm, defquery, defun, defproc, defdist}, %
  morekeywords=[6]{cps-fn, fm, lambda, query, with-primitive-procedures}, %
  morekeywords=[7]{%
    doquery, %
    conditional, %
    collect-by, equalize, exec, infer, log-marginal, print-predicts, %
    rand, rand-int, rand-nth, rand-roulette, stripdown, warmup, %
    ->CRP-process, ->DP-process, ->GP-process, %
    ->bernoulli-distribution, ->beta-distribution, ->binomial-distribution, %
    ->categorical-crp-distribution, ->categorical-distribution, %
    ->categorical-dp-distribution, ->chi-squared-distribution, %
    ->dirichlet-distribution, ->discrete-distribution, %
    ->exponential-distribution, ->flip-distribution, ->gamma-distribution, %
    ->mvn-distribution, ->normal-distribution, ->poisson-distribution, %
    ->sample, ->observe, sample*, observe*, %
    ->uniform-continuous-distribution, ->uniform-discrete-distribution, %
    ->wishart-distribution, CRP, DP, GP, abs, absorb, acos, asin, atan, %
    bernoulli, beta, binomial, categorical, categorical-crp, categorical-dp, %
    cbrt, ceil, chi-squared, cos, cosh, cov, dirichlet, discrete, exp, %
    exponential, flip, floor, gamma, gen-matrix, log, log-gamma-fn, %
    log-mv-gamma-fn, log-sum-exp, map->CRP-process, map->DP-process, %
    map->GP-process, map->bernoulli-distribution, map->beta-distribution, %
    map->binomial-distribution, map->categorical-crp-distribution, %
    map->categorical-distribution, map->categorical-dp-distribution, %
    map->chi-squared-distribution, map->dirichlet-distribution, %
    map->discrete-distribution, map->exponential-distribution, %
    map->flip-distribution, map->gamma-distribution, map->mvn-distribution, %
    map->normal-distribution, map->poisson-distribution, %
    map->uniform-continuous-distribution, map->uniform-discrete-distribution, %
    map->wishart-distribution, mvn, normal, poisson, pow, produce, %
    rint, round, signum, sin, sinh, sqrt, tag, tan, tanh, transform-sample, %
    uniform-continuous, uniform-discrete, wishart, uniform, %
    add-log-weight, add-predict, clear-predicts, get-log-weight, %
    get-mem, get-predicts, in-mem?, set-log-weight, set-mem, %
  }, %
  morekeywords=[8]{factor, mem, observe, predict, retrieve, sample, store}, %
  sensitive, %
  alsoletter={:,-,+,*,?,/,!,>,<,.}, %
  morecomment=[l][\color{gray}]{;}, %
  morestring=[b]", %
}[keywords,comments,strings]
 
\lstdefinestyle{default}{language=Anglican,basicstyle=\ttfamily\small, columns=flexible, showstringspaces=false, numbers=left}
\lstdefinestyle{clojure}{language=Anglican,basicstyle=\ttfamily\small, columns=flexible, showstringspaces=false}

\usepackage[accepted]{icml2020}

\icmltitlerunning{Divide, Conquer, and Combine}

\begin{document}

\twocolumn[
\icmltitle{Divide, Conquer, and Combine: a New Inference Strategy \\ 
	for Probabilistic Programs with Stochastic Support}

\icmlsetsymbol{equal}{*}

\begin{icmlauthorlist}
	\icmlauthor{Yuan Zhou}{ox-cs}
	\icmlauthor{Hongseok Yang}{kaist-cs}
	\icmlauthor{Yee Whye Teh}{ox-stats}
	\icmlauthor{Tom Rainforth}{ox-stats}
\end{icmlauthorlist}

\icmlaffiliation{ox-cs}{Department of Computer Science, University of Oxford, United Kingdom}
\icmlaffiliation{ox-stats}{Department of Statistics, University of Oxford, United Kingdom}
\icmlaffiliation{kaist-cs}{School of Computer Science, KAIST, South Korea}

\icmlcorrespondingauthor{Yuan Zhou}{yuan.zhou@cs.ox.ac.uk}

\icmlkeywords{Probabilistic Programming, Bayesian Inference}

\vskip 0.3in
]

\printAffiliationsAndNotice{}  %

\begin{abstract}

Universal probabilistic programming systems (PPSs) provide a powerful framework for specifying rich
probabilistic models. 
They further attempt to automate the process of drawing inferences from these models, but doing this successfully is severely hampered by the wide range of non--standard models they can express. 
As a result, although one can specify complex models in a universal PPS, the provided inference engines often fall far short of what is required. 
In particular, we show that they produce surprisingly unsatisfactory performance for models where the support varies between executions, often doing no better than importance sampling from the prior. 
To address this, we introduce a new inference framework: Divide, Conquer, and Combine, 
which remains efficient for such models, 
and show how it can be implemented as an automated and generic PPS inference engine. 
We empirically demonstrate substantial performance improvements over existing approaches on three examples.
\end{abstract}

\section{Introduction}
\label{sec:intro}

Probabilistic programming systems (PPSs) provide a flexible platform where probabilistic models are specified as programs and inference procedures are performed in an automated manner.
Some systems, such as BUGS~\cite{tesauro2012bayesian} and Stan~\cite{carpenter2017stan}, are primarily designed around the efficient automation of a small number of inference strategies and the convenient expression of models for which these inference strategies are suitable.

\emph{Universal PPSs}, such as Church~\cite{goodman2012church}, Venture~\cite{mansinghka2014venture}, Anglican~\cite{wood2014new}, and Pyro~\cite{bingham2018pyro}, on the other hand, are set up to try and support the widest possible range of models a user might wish to write.
Though this means that such systems can be used to write models which would be otherwise difficult to encode, this expressiveness comes at the cost of significantly complicating the automation of inference.
In particular, models may contain random variables with mixed types or have varying, or even unbounded, dimensionalities; characteristics which cause significant challenges at the inference stage.

\begin{figure}[!t]
	\centering
	\includegraphics[width=0.4\linewidth]{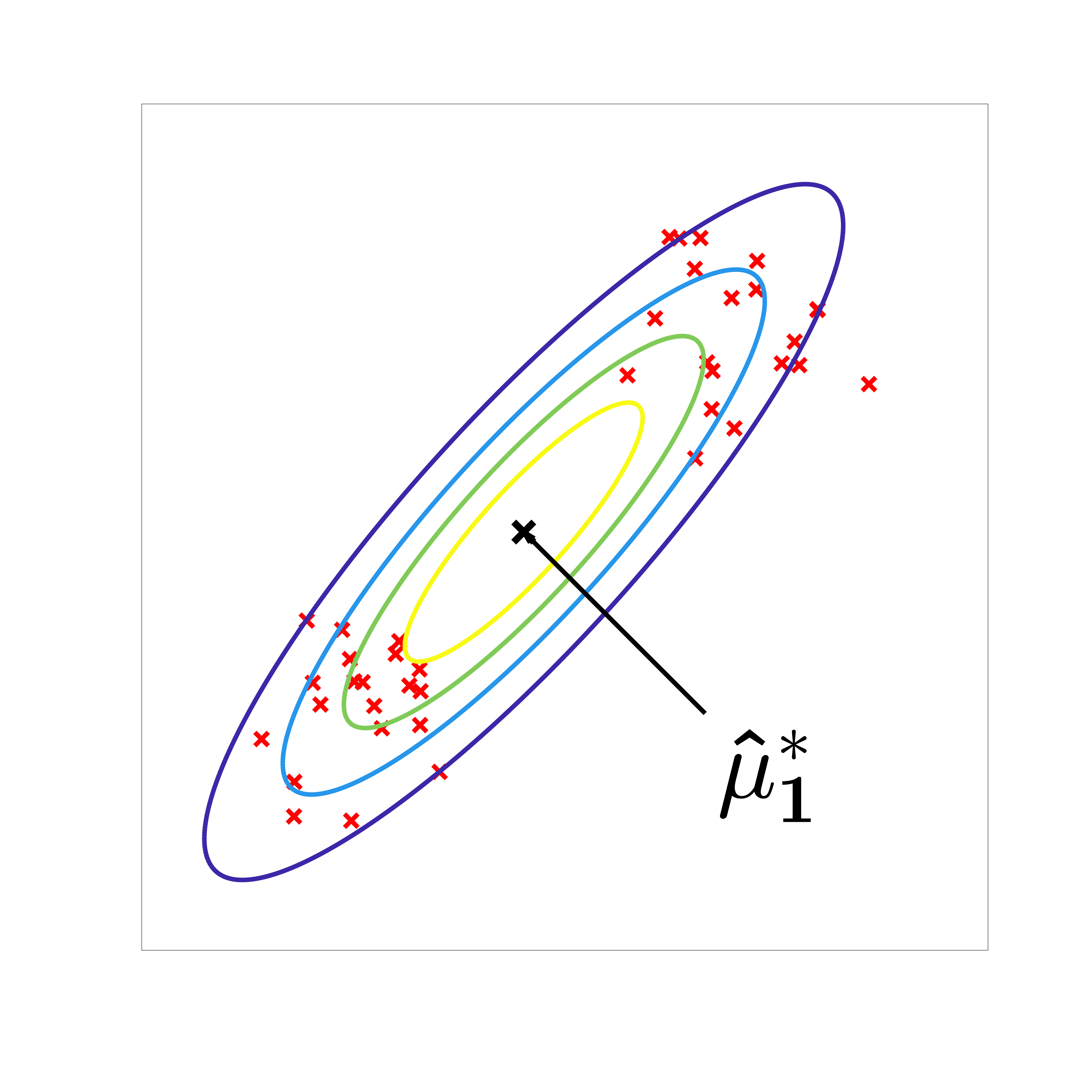}
	\hspace{10pt}
	\includegraphics[width=0.4\linewidth]{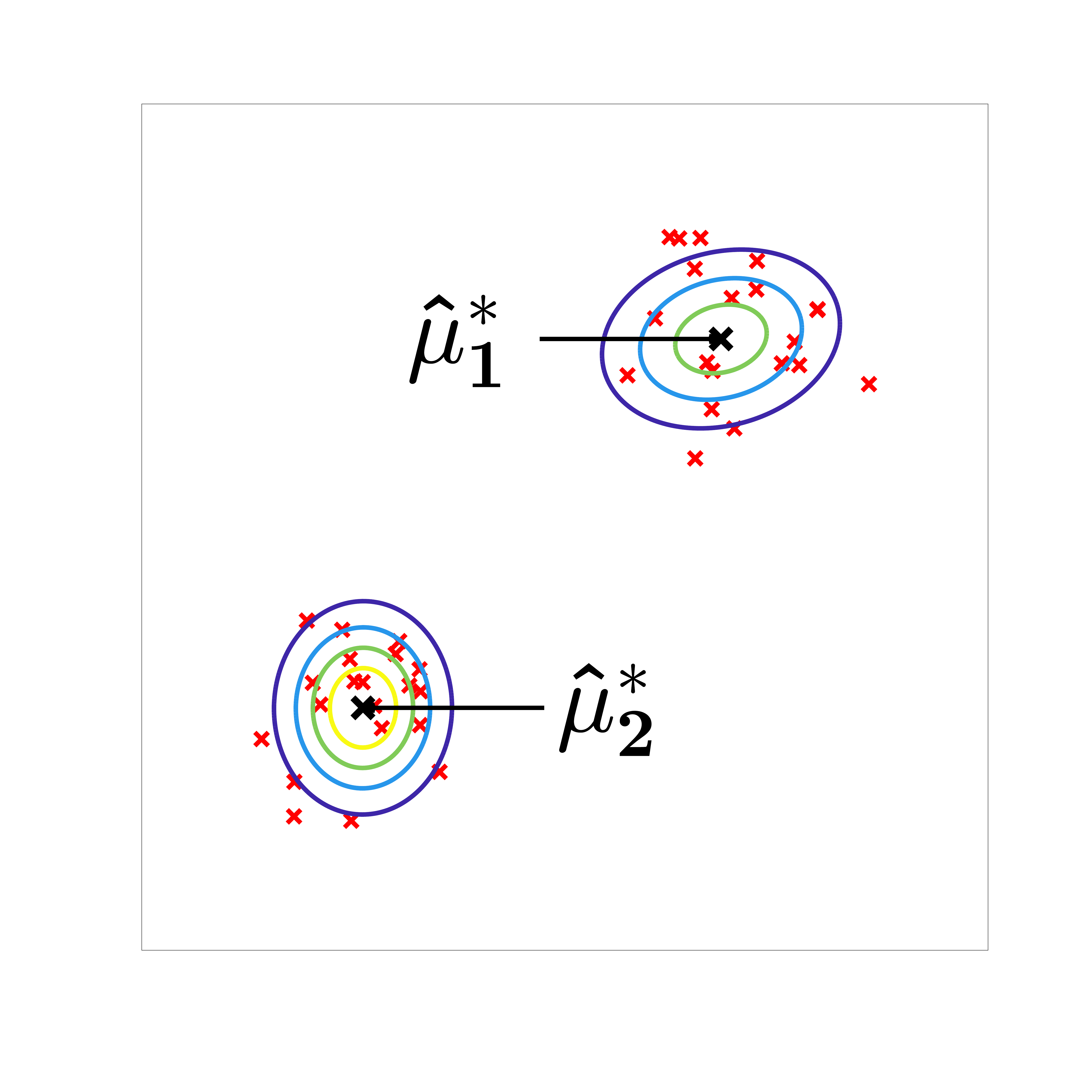}
	\vspace{-10pt}
	\caption{MAP estimate of the means and covariances of a Gaussian mixture model in the cases of $K=1$ and $K=2$ clusters.  If $K$ is itself random, the model has stochastic support as the parameters of the second cluster, e.g.~its mean, only exist when $K=2$. \vspace{-5pt}
	}
	\label{fig:GMM-demo-2}
	\vspace{-5pt}
\end{figure}

In this paper, we aim to address one of the most challenging model characteristics: variables whose very existence is stochastic, often, though not always, leading to the overall dimensionality of the model varying between realizations.
Many practical models posses this characteristic.
For example, many models contain a variable controlling an allowed number of states, such as the number of clusters in a mixture model (see Figure~\ref{fig:GMM-demo-2},~\citealt{richardson1997bayesian, nobile2007bayesian}), or the number of states in a HMM or change point model~\cite{fox2008hdp}.
More generally, many inference problems involve some sort of Bayesian model averaging where the constituent models do not share the exact same set of parameters.
Other models are inherently defined on spaces with non-static support, such as probabilistic context free grammars~(PCFGs)~\cite{manning1999foundations}, program induction models~\cite{perov2014learning}, kernel or function induction models~\cite{schaechtle2016time,janz2016probabilistic}, many Bayesian non-parametric models~\cite{roy2008mondrian,teh2010dirichlet}, and a wide range of simulator--based models~\cite{le2016inference,baydin2019efficient}.

These models can be easily expressed in universal PPSs via branching statements, stochastic loops, or higher-order functions
(see e.g.~Figure~\ref{fig:if-code-trace}).
However, performing inference in them is extremely challenging, with the desire of PPSs to automate this inference complicating the problem further.

A number of automated inference engines have been proposed to provide consistent estimates in such settings~\cite{wingate2011lightweight,wood2014new,tolpin2015maximum,ge2018turing,bingham2018pyro}.
However, they usually only remain effective for particular sub-classes of problems.
In particular, we will show that they can severely struggle on even ostensibly simple problems, often doing no better, and potentially even worse, than importance sampling from the prior.

To address these shortfalls, we introduce a completely new framework---\emph{\textbf{Divide, Conquer, and Combine}} (DCC)---for performing inference in such models.
DCC works by \emph{\textbf{dividing}} the program into separate straight-line sub-programs with fixed support, \emph{\textbf{conquering}} these separate sub-programs using an inference strategy that exploits the fixed support to remain efficient, and then \emph{\textbf{combining}} the resulting sub-estimators to an overall approximation of the posterior.
The main motivation behind this approach is that the difficulty in applying Markov chain Monte Carlo~(MCMC) strategies to such programs lies in the transitioning between variable configurations; within a given configuration efficient inference might still be challenging, but will be far easier than tackling the whole problem directly.
Furthermore, this approach also allows us to introduce meta-strategies for allocating resources between sub-programs, thereby explicitly controlling the exploration-exploitation trade-off for the inference process in a manner akin to~\citet{rainforth2018inference,lu2018exploration}. 
To demonstrate its potential utility, we implement a specific realization of our DCC framework as an automated and general-purpose inference engine in the PPS Anglican~\cite{wood2014new}, finding that it is able to achieve substantial performance improvements and tackle more challenging models than existing approaches.

To summarize, our key contributions are as follows:
\vspace{-3\topsep}
\begin{itemize}[noitemsep]
\item We highlight shortcomings of existing PPS inference engines in the presence of stochastic support;
\item We introduce a new framework, DCC, for performing inference in such problems;
\item We implement a specific realization of this DCC framework in the PPS Anglican.
\end{itemize}
\vspace{-\topsep}

\vspace{-10pt}
\section{Probabilistic Programming}
\label{sec:background-trace}

\begin{figure}[!t]
	\centering
	\vspace{-9pt}
\hspace{5pt}\begin{minipage}{0.4\textwidth}
		\begin{subfigure}
			\centering
			\input{figures/branching-code-example}
		\end{subfigure}%
	\end{minipage}
	\begin{minipage}{0.47\textwidth}
		\begin{subfigure}
			\centering
			\includegraphics[height=0.4\textwidth, width=\textwidth]{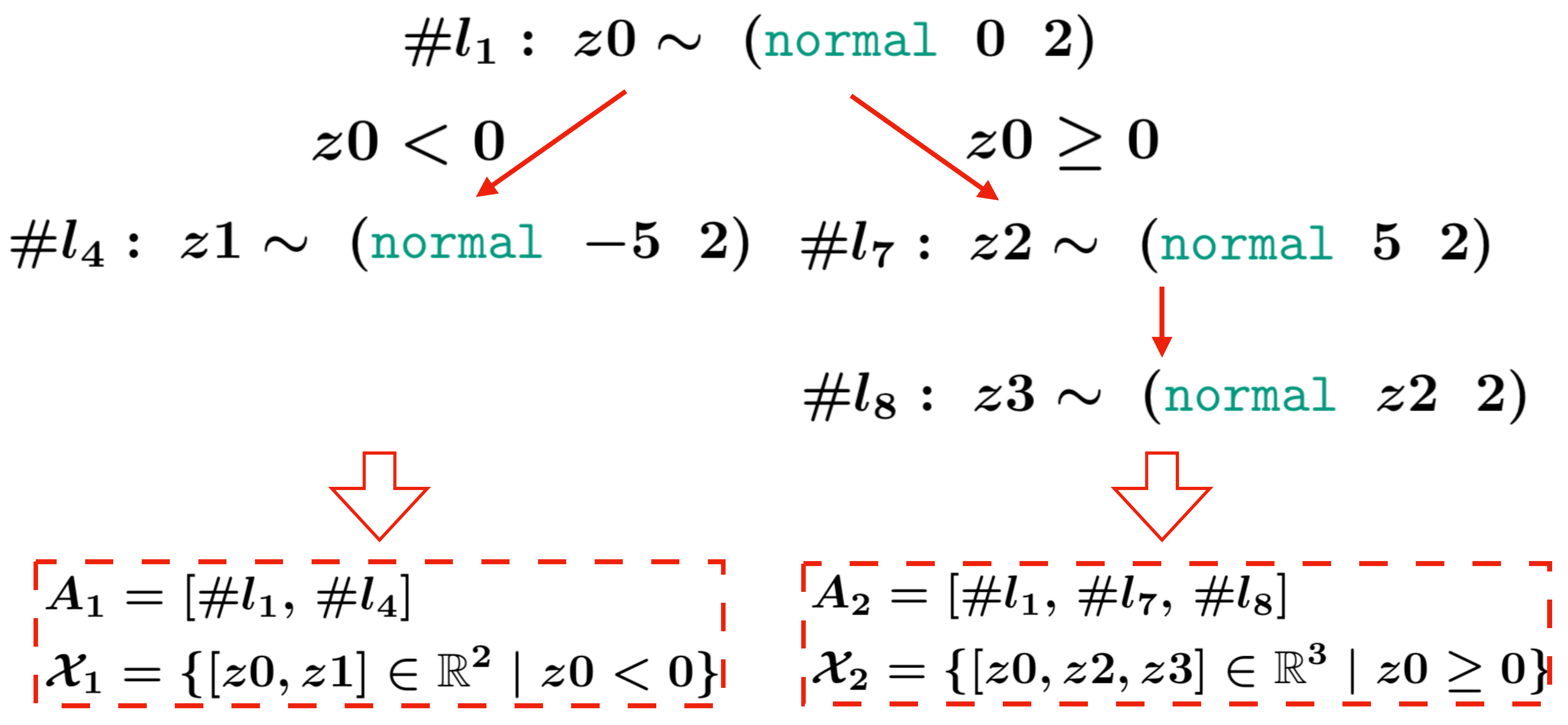}
		\end{subfigure}
	\end{minipage}
	\vspace{-12pt}
	\caption{Example program with stochastic support (top) and its possible execution traces (bottom).
		The two branches of the \lstinline{if} each produce a different sample path, denoted as $A_1$ and $A_2$, and each have a different supports denoted as $\mathcal{X}_1$ and $\mathcal{X}_2$ respectively.  }
	\label{fig:if-code-trace}
	\vspace{-6pt}
\end{figure}

In universal PPSs~\cite{goodman2012church,mansinghka2014venture,wood2014new,bingham2018pyro,ge2018turing}, a program may denote a model with varying support: different realizations of a program may lead to different sets of variables being sampled.  For example, in Figure~\ref{fig:if-code-trace}, there are two possible variable configurations, $[\texttt{z0},\texttt{z1}]$ and $[\texttt{z0},\texttt{z2},\texttt{z3}]$, due to the stochastic control-flow. %

To aid exposition and formalize these programs, 
we will use the formalization of
the particular universal PPS Anglican~\cite{wood2014new, tolpin2016design}, but note that our ideas are applicable to other universal PPSs.
Anglican extends the syntax of Clojure with
two special forms: \lstinline{sample} and \lstinline{observe}, between which the distribution of the program is defined.
\lstinline{sample} statements are used to draw random variables, while \lstinline{observe} statements are used to condition on data.
Informally, they can be thought of as prior and likelihood terms, respectively.

The density of an Anglican program is derived by executing it in a forward manner, drawing from \lstinline{sample} statements when encountered, and keeping track of density components that originate from the \lstinline{sample} and \lstinline{observe} terms.
Specifically, let $\{x_i\}_{i=1}^{n_x} = (x_1, \dots, x_{n_x})$ represent the random variables generated from the encountered \lstinline{sample} statements, where the $i^\text{th}$ statement among them has lexical program address $a_i$, input $\eta_i$, and density $f_{a_i}(x_i|\eta_i)$.
Analogously, let $\{y_j\}_{j=1}^{n_y} = (y_1, \dots, y_{n_y})$ represent the observed values of the $n_y$ \lstinline{observe} statements encountered during execution, which have lexical addresses $b_{j}$ and corresponding densities $g_{b_j}(y_j|\phi_j)$, where $\phi_j$ is analogous to $\eta_i$.
The density is now given by $\pi(x) = \gamma(x)/Z$ where
\vspace{-4pt}
\begin{align}
\label{eq:density-anglican-program}
\gamma(x) &:=
\prod_{i=1}^{n_x} f_{a_i}(x_i|\eta_i)\prod_{j=1}^{n_y} g_{b_j}(y_j|\phi_j), \\
\label{eq:Z-anglican-program}
Z &:=
\int \prod_{i=1}^{n_x} f_{a_i}(x_i|\eta_i)\prod_{j=1}^{n_y} g_{b_j}(y_j|\phi_j) dx_{1:{n_x}},
\end{align}
and the associated reference measure is implicitly defined through the executed \lstinline{sample} statements.
Note that everything here (i.e. $n_x$, $n_y$, $x_{1:n_x}$, $y_{1:n_y}$, $a_{1:n_x}$, $b_{1:n_y}$, $\eta_{1:n_x}$, and $\phi_{1:n_y}$) is a random variable, but each is deterministically calculable given $x_{1:n_x}$ (see \S4.3.2 of \citet{tom-thesis}).

From this, we see that it is sufficient to denote an \emph{\textbf{execution trace}} (i.e.~realization) of an Anglican program by the sequence of the addresses of the encountered \lstinline{sample} statements and the corresponding sampled values, namely  $[a_i, x_i]_{i=1}^{n_x}$.\footnote{Strictly speaking, the addresses $a_i$ can be deterministically derived from the sampled values for a given program.  However, for our purposes it will be convenient to think about first sampling the path $a_{1:n_x}$ and then sampling $x_{1:n_x}$ conditioned on this path.}
For clarity, we refer to the sequence $a_{1:n_x}$ as the \emph{\textbf{path}} of an execution trace and $x_{1:n_x}$ as the \emph{\textbf{draws}}.
A program with \emph{\textbf{stochastic support}} can now be more formally defined as one for which the path $a_{1:n_x}$ varies between realizations: different values of the path correspond to different \emph{{\textbf{configurations}}} of variables being sampled.
\vspace{-2pt}

\section{Shortcomings of Existing Inference Engines}
\label{sec:inf-eng}
\vspace{-2pt}

In general, existing inference engines that can be used for (at least some) problems with stochastic support can be grouped into five categories: 
importance/rejection sampling, 
particle based inference algorithms (e.g.~SMC, PG, PIMH, PGAS, IPMCMC, RM-SMC, PMMH, SMC$^2$), 
MCMC approaches with automated proposals (e.g.~LMH, RMH), 
MCMC approaches with user--customized proposals (e.g.~RJMCMC), and variational approaches (VI, BBVI).
More details on each are provided in Appendix~\ref{sec:supp-discussion-PPSs}.

Importance/rejection sampling approaches can straightforwardly be applied in stochastic support settings by using the prior as the proposal, but their performance deteriorates rapidly as the dimensionality increases. 
Particle--based approaches offer improvements for models with sequential structure~\citep{wood2014new}, but become equivalent to importance sampling when this is not the case.

Some variational inference (VI) approaches can, at least 
in theory, be used in the presence of stochastic support~\citep{bingham2018pyro,cusumano2019gen}, but this can require substantial problem--specific knowledge to construct a good guide function, 
requires a host of practical issues to be overcome~\citep[Chapter 6]{paige2016automatic}, and
is at odds with PPSs desire for automation. 
Furthermore, these approaches produce inconsistent estimates and current engines often give highly unsatisfactory practical performance as we will show later.
When using such methods, it is common practice to side--step the complications originating from stochastic support by approximating the model with a simpler one with fixed support.
For example, though Pyro can ostensibly support running VI in models with stochastic support, their example implementation of a Dirichlet process mixture model ({\small \url{https://pyro.ai/examples/dirichlet_process_mixture.html}}) uses fixed support approximations by assuming a bounded number of mixture components before stripping clusters away. %

Because of these issues, arguably the current go--to approaches for programs with stochastic support are specialized MCMC approaches~\cite{wingate2011lightweight,roberts2019reversible}.
However, these are themselves far from perfect and, as we show next, they often given performance far worse than one might expect.

\begin{figure}[!t]
	\centering
	\includegraphics[width=\linewidth]{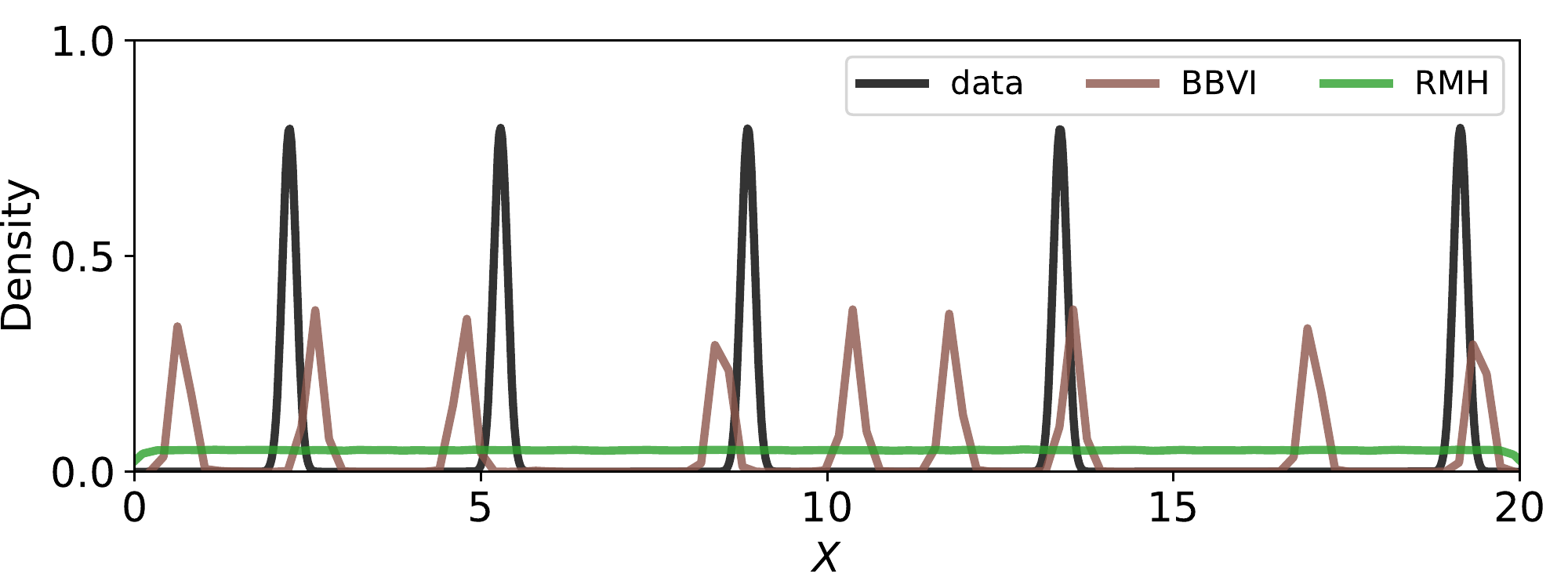}
	\vspace{-20pt}
	\caption{Kernel density estimation of the synthetic data (black) of the univariate mixture model and the posterior predictive distribution from BBVI (brown) and RMH (green) in Anglican.
	}
	\label{fig:GMM-demo}
	\vspace{-5pt}
\end{figure}
\vspace{-8pt}

\paragraph{A demonstrative example} 
Consider the following simple mixture model with an unknown number of clusters $K$
\begin{align*}
\vspace{-10pt}
K &\sim \; \mathrm{Poisson}(9){+}1, \;
\mu_{k} \sim \mathrm{Uniform}\bigg({\frac{20(k-1)}{K}},\, \frac{20k}K \bigg), \displaybreak[0]
\\ 
z_{n} &\sim \mathrm{Cat}(\{1/K,...,1/K\}),\;
 y_n  \sim  \mathcal{N}(\mu_{z_n}, 0.1^2).  
  \vspace{-10pt}
\end{align*}
Here $\mu_{1:K}$ are the cluster centers, 
$z_{1:N}$ are the cluster assignments, and
$y_{1:N}$ is the observed data.
When conducting inference, we can analytically marginalize out the cluster assignments $z_{1:N}$ and perform inference on $K$ and $\mu_{1:K}$ only.
However, as the prior on $K$ is a Poisson distribution, the number of parameters in the model is unbounded.
Using the model itself, we generated a synthetic dataset of $y_{1:150}$ for an one-dimensional mixture of five clusters (i.e. $K=5$).

We now wish to perform inference over both the number of clusters $K$ and the cluster means $\mu_{1:K}$, so that we can make predictions from the posterior predictive distribution.
Two approaches we might try are VI, for which we use Anglican's black-box variational inference~(BBVI) implementation~\citep{ranganath2014black,paige2016automatic}, and MCMC, for which we take its RMH algorithm~\cite{le2015rmh}, a derivative of the single-site MH algorithm of~\citep{wingate2011lightweight} (see below).
Unfortunately, as we see in Figure~\ref{fig:GMM-demo}, both approaches fail spectacularly and produce posterior predictive distributions that bare little resemblance to the data.
In particular, they fail to properly encapsulate the number of clusters.
As we will show in \S\ref{sec:gmm}, importance sampling and particle--based inference engines fare no better for this problem.
In fact, we are not aware of any fully \emph{automated} inference engine that is able to give satisfactory performance, despite the apparent simplicity of the problem.

\subsection{Why is MCMC so Hard with Stochastic Support?} 
To run MCMC on a program with stochastic support, one needs to be able to construct a transition kernel that is able to switch between the configurations; many popular MCMC methods, like  Hamiltonian Monte Carlo~(HMC), cannot be applied.
One can either look to construct this transition kernel manually through a carefully chosen user--specified trans--dimensional proposal and then using a reversible jump MCMC~(RJMCMC) scheme~\cite{green1995reversible, green2003trans,roberts2019reversible,cusumano2019gen}, 
or use a general--purpose kernel that works on all models~\cite{wingate2011lightweight, dippl}. %

The predominant approaches for the latter are the single-site MH (a.k.a.~LMH) algorithm~\cite{wingate2011lightweight} and its extensions~\cite{yang2014generating,le2015rmh,TolpinMPW15,RitchieSG16}.
LMH is based around a Metropolis-within-Gibbs (MwG) approach on the program traces~\cite{brooks2011handbook}, whereby one first samples a variable in the execution trace, ${i\in1:n_x}$, uniformly at random and then proposes a MwG transition to this variable, $x_i \to x_i'$, followed by an accept/reject step.
Anglican's RMH is a particular case of this LMH approach where the proposal is a mixture of resampling $x_i$ from the prior $f_{a_i}(x_i|\eta_i)$ %
and a local random walk proposal $p(x'_i | x_i)$.

The problem with LMH approaches is that 
if the transition of $x_i$ influences the downstream control flow of the program, the downstream draws no longer produce a valid execution trace and so must be redrawn, typically using the prior.
This can cause the mixing of the sampler over configurations to become extremely slow; the need to transit between configurations bottlenecks the system. 

This problem is also far from specific to the exact transition kernel used by LMH samplers: it is also extremely challenging to hand--craft RJMCMC proposals to be effective.
Namely, proposing changes in the configuration introduces new variables that might not be present in the current configuration, such that our proposal for them effectively becomes an importance sampling proposal.
Furthermore, the posterior on the other variables may shift substantially when the configurations changes.

In short, one loses a notion of locality: having a sample in a high density region of one configuration typically provides little information about which regions have a high density for another configuration.
For example, in a mixture model shown in Figure~\ref{fig:GMM-demo-2}, having a good characterization of $\mu_1 | {K=1}$ provides little information about the distribution of $\mu_1 | {K=2}$, as shown by the substantial change in their mode, $\mu_1^*$. %
It is thus extremely difficult to design proposals which maintain a high acceptance rate when proposing a new configuration: once in a high density region of one configuration, it becomes difficult to switch to another configuration.
This problem is further compounded by the fact that RJMCMC only estimates the relative mass of each configuration through the relative frequency of transitions, giving a very slow convergence for the overall sampler.

\section{Divide, Conquer, and Combine}
\label{sec:method}

The challenges for running MCMC methods on programs with stochastic support stem from the difficultly in transitioning {\emph{between}} configurations properly.
To address this, we now introduce a completely new inference framework for these programs:
\emph{\textbf{Divide, Conquer, and Combine}}~(DCC).
Unlike most existing inference approaches which directly target the full program density (i.e.~\eqref{eq:density-anglican-program}), DCC breaks the problem into individual sub-problems with \emph{\textbf{fixed}} support and tackles them separately.
Specifically, it \emph{\textbf{divides}} the overall program into separate straight-line sub-programs according to their execution paths, \emph{\textbf{conquers}} each sub-program by running inference locally,
and \emph{\textbf{combines}} the results together to form an overall estimate
in a principled manner. 

In doing this, DCC transfers the problem of designing an MCMC proposal which both efficiently transitions between paths (i.e.~varying configurations) and mixes effectively over the draws on that path, to that of a) performing inference locally over the draws of each given path, and b) learning the relative marginal posterior mass of these paths.
This separation brings the benefit that the inference for a given path can typically be performed much more efficiently than when using a global sampler, as it can exploit the fixed support and does not need to deal with changes in the variable configuration.
Furthermore, it allows the relative posterior mass to be estimated more reliably than with global MCMC schemes, for which this is estimated implicitly through the relative frequency of the, typically infrequent, transitions.

We now explain the general setup for each component of DCC.
Specific strategies for each will be introduced in \S\ref{sec:method-PPS}, while an overview of the approach is given in Algorithm~\ref{alg:dcc}.

\subsection{Divide}
\label{sec:method-divide}
The aim of DCC's divide step is to split the given probabilistic program into its constituent straight-line programs (SLPs), where each SLP is a partition of the overall program corresponding to a particular sequence of sample addresses encountered during execution, i.e.~a particular path $a_{1:n_{x}}$.
Each SLP has a fixed support as the set of variables it draws are fixed by the path, i.e.~the program draws from the same fixed set of \lstinline|sample| statements in the same order.

Introducing some arbitrary indexing for the set of SLPs, we use $A_k$ to denote the path for the $k^{\text{th}}$ SLP (i.e.~$a_{1:n_{x,k}}=A_k$ for every possible realization of this SLP).
The set of of all possible execution paths is now given by $A = \{A_k\}_{k=1}^K$, where $K$ must be countable (but need not be finite).
For the example in Figure~\ref{fig:if-code-trace}, this set consists of two paths $A_1 = [\#l_1,\#l_4]$ and $A_2 = [\#l_1, \#l_7, \#l_8]$, where we use $\#l_j$ to denote the lexical address of the \lstinline|sample| statement is on the $j^{\text{th}}$ line. 
Note that, for a given program, each SLP is uniquely defined by its corresponding path $A_k$; we will sometimes use $A_k$ to denote an SLP.

Dividing a program into its constituent SLPs implicitly partitions the overall target density into disjoint regions,
with each part defining a sub-model on the corresponding sub-space.
The unnormalized density $\gamma_{k}(x)$ of the SLP $A_k$
is defined with respect to the variables $\{x_i\}_{i=1}^{n_{x,k}}$ that are paired with the 
addresses $\{a_i\}_{i=1}^{n_{x,k}}$ of $A_k$ (where we have used the notation $n_{x,k}$ to emphasize that this is now fixed).
We use $\mathcal X_k$ to denote its corresponding support.
Note that the union of all the $\mathcal X_k$ is the support of the original program,
$\mathcal{X} = \bigcup_{k=1}^K  \mathcal X_k $.
Analogously to~\eqref{eq:density-anglican-program}, we now have that the density of SLP $k$ is 
$\pi_k(x) = \gamma_{k}(x)/Z_k$  where
\begin{align}
\vspace{-10pt}
\gamma_{k}(x) 
:=& \, \gamma(x){\mathbb{I}[x\in \mathcal X_k]} \nonumber \\
=& \, \mathbb{I}[x \in \mathcal X_k] \prod_{i=1}^{n_{x,k}} f_{A_k[i]}(x_i|\eta_i)\prod_{j=1}^{n_{y}} g_{b_j}(y_j|\phi_j),
\label{eq:density-straightline-program}\\
Z_k :=& \, \int_{x\in \mathcal{X}_k} \gamma_{k}(x) dx.
\label{eq:Z-straightline-program}
\end{align} 
Unlike for~\eqref{eq:density-anglican-program}, $n_{x,k}$ and $A_k$ are now, critically, deterministic variables so that the support of the problem is fixed.
Though $b_j$ and $n_{y}$ may still be stochastic, they do not effect the reference measure of the program 
and so this does not cause a problem when trying to perform MCMC sampling.

Following our example in Figure~\ref{fig:if-code-trace}, for $A_1$ we have $x_{1:2} = [z0, z1]$, $\mathcal{X}_1\!=\!\{[x_1, x_2]\!\in\!\mathbb{R}^2\!\mid\!{x_1<0}\}$, and $ \gamma_{1}(x) = \mathcal{N}(x_1;0,2)\mathcal{N}(x_2;-5,2)\mathcal{N}(y_1;x_2,2)\mathbb{I}[x_1\,{<}\,0]$.
For $A_2$, we instead have $x_{1:3}=[z0, z2, z3]$, $\mathcal{X}_2 = \{[x_1, x_2, x_3] \in \mathbb{R}^3 \mid x_1 \geq 0\}$ 
and $ \gamma_{2}(x) = \mathcal{N}(x_1;0,2)\mathcal{N}(x_2;5,2)\mathcal{N}(x_3;x_2,2)\mathcal{N}(y_1;x_3,2)\mathbb{I}[x_1\!\geq\!0]$.

To implement this divide step, we now need a mechanism for establishing the SLPs.
This can either be done by trying to extract them all upfront, or by dynamically discovering them as the inference runs, see \S\ref{sec:method-PPS-divide}.

\subsection{Conquer}
\label{sec:method-conquer}

Given access to the different SLPs produced by the divide step, DCC's conquer step looks to carry out the %
local inference for each.
Namely, it aims to produce a set of estimates for the individual SLP densities $\pi_{k}(x)$ %
and the corresponding marginal likelihoods $Z_k$.
As each SLP has a fixed support, this can be achieved with conventional inference approaches, with a large variety of methods potentially suitable.
Note that $\pi_{k}(x)$ and $Z_k$ need not be estimated using the same approach, e.g. we may use an MCMC scheme to estimate $\pi_{k}(x)$ and then introduce a separate estimator for $Z_k$.
One possible estimation strategy is given in~\S\ref{sec:method-PPS-conquer-estimator}.

An important component in carrying out this conquer step effectively is to note that it is not usually necessary to obtain estimates of equally-high fidelity for all SLPs.
Specifically, SLPs with small marginal likelihoods $Z_k$ only make a small contribution to the overall density and thus do not require as accurate estimation as SLPs with large $Z_k$.
As such, it will typically be beneficial to carry out \emph{\textbf{resource allocation}} as part of the conquer step, that is, to generate our estimates in an online manner where at each iteration we use information from previous samples to decide the best SLP(s) to update our estimates for.
See~\S\ref{sec:method-PPS-conquer-resource} for one possible such strategy.

\subsection{Combine}
\label{sec:method-combine}

The role of DCC's combine step is to amalgamate the local estimates from the individual SLPs to an overall estimate of the  distribution for the original program.
For this, we can simply note that, because the supports of the individual SLPs are disjoint and their union is the complete program, we have
$\gamma(x) = \sum_{k=1}^{K} \gamma_{k}(x)$ and $Z = \sum_{k=1}^{K} Z_k$,
such that the unnormalized density and marginal likelihoods are both additive.
Consequently, we have 
\vspace{-2pt}
\begin{align}
\pi(x) 
&= \frac{\sum_{k=1}^K{\gamma_{k}(x)}}{\sum_{k=1}^{K} Z_k} 
=\frac{\sum_{k=1}^K{Z_k \pi_{k}(x)}}{\sum_{k=1}^{K} Z_k} \nonumber \\
&\approx
\frac{\sum_{k=1}^K{\hat{Z}_k \hat{\pi}_{k}(x)}}{\sum_{k=1}^{K} \hat{Z}_k} =: \hat{\pi}(x)
\label{eq:comb_est}
\end{align}
where $\hat{\pi}_{k}(x)$ and $\hat{Z}_k$ are the SLP estimates generated during the conquer step.
Note that, by proxy, this also produces the overall marginal likelihood estimate
$\hat{Z} := \sum_{k=1}^{K}  \hat{Z}_k$.

\begin{algorithm}[t]
	\caption{Divide, Conquer, and Combine (DCC) \label{alg:dcc}}
	\begin{algorithmic}[1]
		\floatname{algorithm}{Procedure}
		\renewcommand{\algorithmicrequire}{\textbf{Input:}}
		\renewcommand{\algorithmicensure}{\textbf{Output:}}	
		\renewcommand{\COMMENT}[1]{\hfill$\triangleright$ {#1}}
		
		\REQUIRE
		Program $\boldsymbol{\prog}$, number of iterations $T$
		\ENSURE 
		Posterior approx $\hat{\pi}$, ML estimate $\hat{Z}$
		
		\STATE Obtain initial set of discovered SLPs $\hat{A}$ \COMMENT{ \S\ref{sec:method-divide}, \S\ref{sec:method-PPS-divide}}
		\STATE Compute initial estimates $\forall A_k\in\hat{A}$%
		\COMMENT{\S\ref{sec:method-conquer}, \S\ref{sec:method-PPS-conquer-estimator}}
		\FOR{$t=1,\dots,T$}
		\STATE 
		Choose an SLP $A_{k} \in \hat{A}$ to update
		\COMMENT{\S\ref{sec:method-PPS-conquer-resource}}
		
		\STATE Update local estimates $\hat\pi_{k}$ and $\hat Z_{k}$
		\COMMENT{\S\ref{sec:method-conquer}, \S\ref{sec:method-PPS-conquer-estimator}}
		
		\STATE  {[Optional]} Look for undiscovered SLPs (e.g.~using a global proposal), add any found to $\hat{A}$ 
		\COMMENT{\S\ref{sec:method-PPS-divide}}
		
		\ENDFOR
		
		\STATE Combine local approximations as per~\eqref{eq:comb_est}
		\COMMENT{\S\ref{sec:method-combine}}
	\end{algorithmic}
\end{algorithm}

When using an MCMC sampler for $\pi_k(x)$, $\hat{\pi}_k(x)$ will take the form of an empirical measure comprising of a set of samples, i.e. $\hat{\pi}_k(x) = \frac{1}{N_k} \sum_{m=1}^{N_k} \delta_{\hat{x}_{k,m}}(x)$.
If we use an importance sampling or particle filtering based approach instead, our empirical measure will compose of weighted samples.
We note that in this case, the $\hat{Z}_k$ term in the numerator of~\eqref{eq:comb_est} will cancel with any potential self-normalization term used in $\hat{\pi}_{k}(x)$, such that we can think of using the estimate
$\pi(x) \approx (\sum_{k=1}^K{ \hat{\gamma}_{k}(x)})/ (\sum_{k=1}^{K} \hat{Z}_k)$.

\section{Theoretical Correctness}
\label{sec:theory}

We now demonstrate that the outlined general DCC approach is consistent (as $T\to\infty$ where $T$ is the number of iterations) given some simple assumptions about the individual component strategies.
At a high level, these assumptions are that the estimators used for each SLP, $\hat{\pi}_k$ and $\hat{Z}_k$, are themselves consistent, we use an SLP extraction strategy that will eventually uncover all of the SLPs with finite probability mass, and our resource allocation strategy selects each SLP infinitely often given an infinite number of iterations.

More formally we have the following result
\begin{restatable}{theorem}{constistency}
	If Assumptions~\ref{assumption1}-\ref{assumption5} in Appendix~\ref{sec:app:theory} hold, then the empirical measure, $\hat{\pi}\left(\cdot\right)$, produced by DCC converges weakly to the conditional distribution of the program in the limit of large number of iterations $T$ :
	\vspace{-5pt}
	$$
	\hat{\pi}\left(\cdot\right) \dto \pi \left(\cdot\right) \quad \text{as} \quad T\to\infty.
	$$
\end{restatable}
\vspace{-5pt}
The proof is provided in Appendix~\ref{sec:app:theory}.
We note that it is typically straightforward to ensure that these assumptions hold; the specific approaches we outline next satisfy them.

\section{DCC in Anglican}
\label{sec:method-PPS}

We now outline a particular realization of our DCC framework.
It is implemented in Anglican and can be used to run inference automatically for any valid Anglican program.
As part of this, we suggest particular strategies for the individual components left unspecified in the last section, but emphasize that these are far from the only possible choices; DCC should be viewed more as a general framework.
Additional details including a complete algorithm block are given in the appendices.

\subsection{Local Estimators}
\label{sec:method-PPS-conquer-estimator}

Recall that the goal for the local inference is to estimate the local target density $\pi_k(x)$ 
and the local marginal likelihood $Z_k$.
Straightforward choices include (self-normalized) importance sampling and SMC as both return a marginal likelihood estimate $\hat{Z}_k$.
However, knowing good proposals for these a priori is challenging and, as we discussed in \S\ref{sec:inf-eng}, na\"{i}ve choices 
are unlikely to perform well.

Thankfully, each SLP has a fixed support, which means many of the complications that make inference challenging for universal PPSs no longer apply.
In particular, we can use conventional MCMC samplers---such as MH, HMC, or MwG---to approximate $\pi_k(x)$.
Due to %
 the fact that individual variable types may be unknown or not even fixed, we have elected to use MwG in our implementation, but note that more powerful inference approaches like HMC may be preferable when they can be safely applied.
To encourage sample diversity and assist in estimating $Z_k$ (see below), we further run $N$ independent MwG samplers for each SLP.

As MCMC samplers do not directly provide an estimate for $Z_k$, we must introduce a further estimator that uses these samples to estimate it.
For this, we use PI-MAIS~\cite{martino2017layered}.
Details are given in Appendix~\ref{sec:app:local-inf}.

\subsection{Discovering SLPs}
\label{sec:method-PPS-divide}

To divide a given model into its constituent sub-models expressed by SLPs, we need a mechanism for discovering them automatically. 
One possible approach would be to analyze the source code of the program using static analysis techniques~\cite{chaganty2013efficiently, nori2014r2}, 
thereby extracting the set of possible execution paths of the program at compilation time. 
Though potentially a viable choice in some scenarios, this can be difficult to achieve for all possible programs in a universal PPS. For example, the number of possible paths may be unbounded.
We therefore take an alternative approach that discovers SLPs dynamically at run--time as part of the inference.
In general, it maintains a set of SLPs encountered so far, and \emph{remembers} any new SLP discovered by the MCMC proposals at each iteration.

Our approach starts by executing the program forward for $T_0$ iterations to generate some sample execution traces from the prior. 
The paths traversed by these sampled traces are recorded, and our set of SLPs is initialized as that of these recorded paths.
At subsequent iterations, after each local inference iteration, 
we then perform one \emph{{global}} MCMC step
based on our current sub-model and trace, producing a new trace with path $A_{k'}$ that may or may not have changed. 
If $A_{k'}$ corresponds to an existing SLP, this sample is discarded (other than keeping count of the number of proposed transitions into each SLP).
However, if it corresponds to an unseen path, it is added to our set of SLPs as a new sub-model,
followed by $T_w$ MCMC steps restricted to that path to burn-in. 
The sample of the final step will be stored as initiation for the future local inference on that path.

\begin{figure*}[t]
	\centering
	\includegraphics[width=0.38\textwidth,align=c]{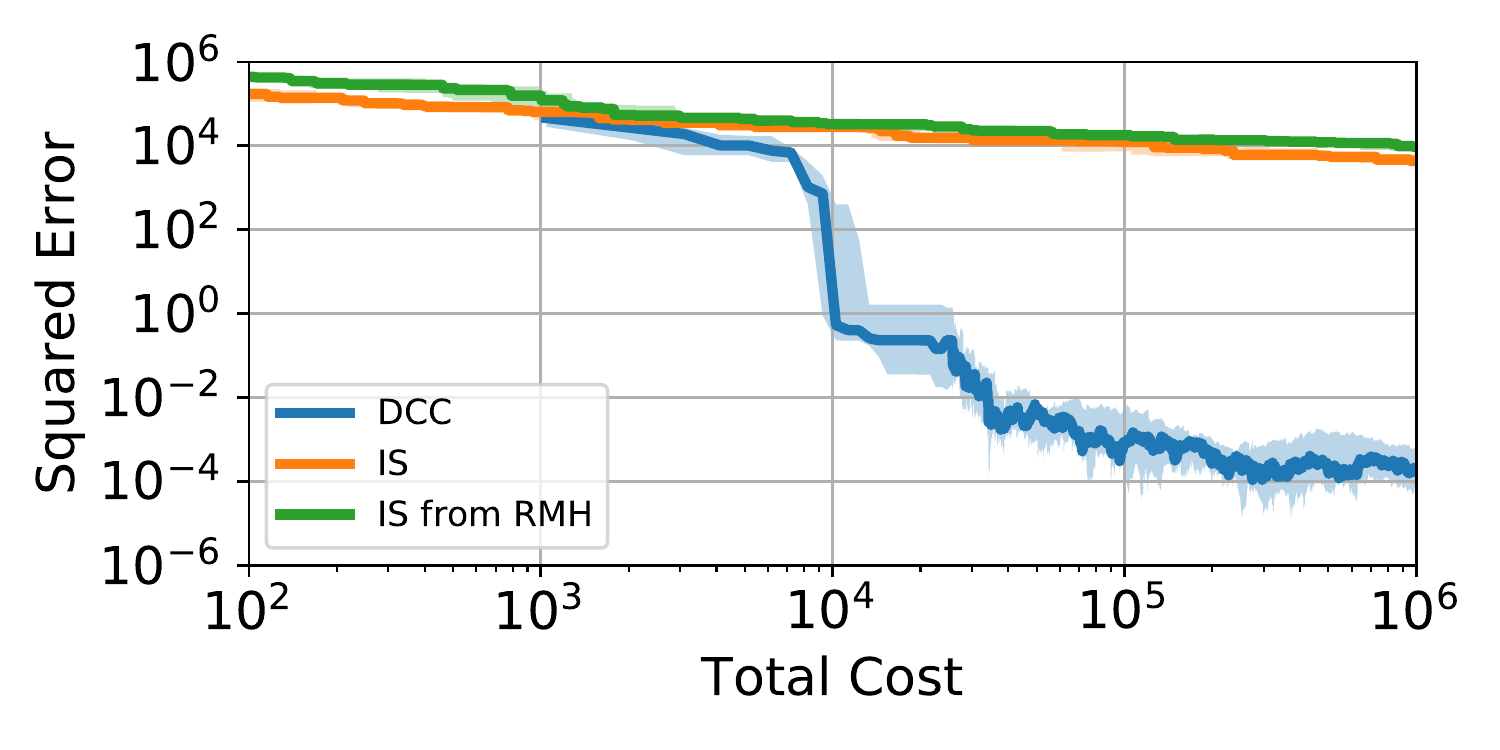}
	~~~~~~~~~~~~
\includegraphics[width=0.41\textwidth,align=c]{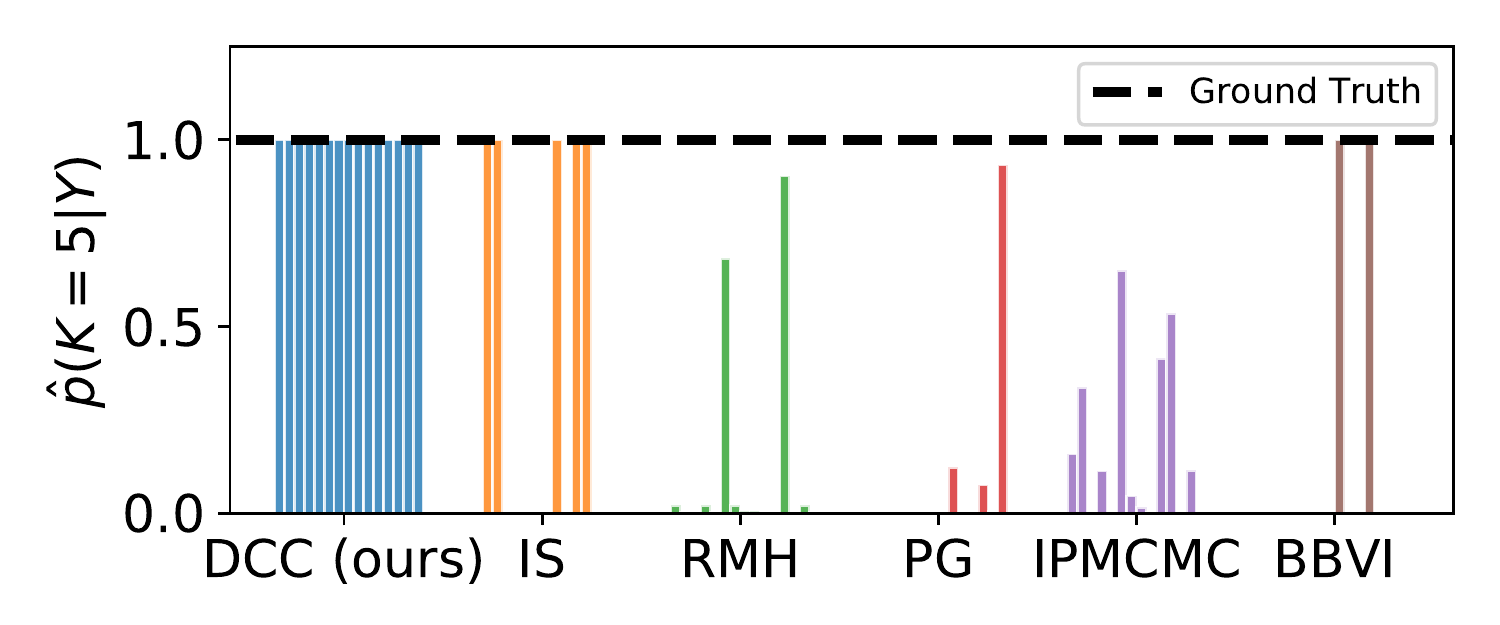}
	\caption{
		Results for DCC and baselines  for the GMM example outlined in \S\ref{sec:inf-eng}.
		[Left] Convergence in squared error in log marginal likelihood estimate ${|\!|}\log \hat{Z} - \log Z_{\mathit{true}}{|\!|}^2$.
		The solid line corresponds to the median across $15$ runs and the shading region $25\%-75\%$ quantiles.
		Note that none of IS, RMH, PG, IPMCMC, and BBVI provide such at estimate, hence their omission; an additional baseline of drawing importance samples from a proposal centered on and RMH chain was considered instead.
		[Right] Final estimates for  $p(K = 5\,|\,y_{1:N_y})$ for each of the $15$ runs, for the ground truth is roughly $0.9998$.
		In both cases, the ground truth was estimated using a very large number of importance samples with a manually adapted proposal.
		We see that DCC substantially outperforms the baselines.
	}
	\label{fig:GMM}
\end{figure*}

The key difference between our strategy and running a single global MCMC sampler (e.g. RMH), is that we do not need this new sample to be \emph{{accepted}} for the new SLP to be ``discovered''. 
Because, as we explained in \S~\ref{sec:inf-eng}, making effective proposals into a high--density region of a new configuration is very challenging, it is unlikely that the new trace sample we propose has high density: even if it corresponds to a path with large $Z_k$, we are unlikely to immediately sample a good set of draws to accompany it.
Therefore, the global MCMC sampler is very likely to miss or \emph{\textbf{forget}} new SLPs since it accepts/rejects movements based on one sample.

DCC, on the other hand, overcomes this problem by first \emph{\textbf{remembering}} the path information of any newly proposed SLP, 
and then carrying out a few local warm-up iterations, before deciding whether an SLP is a promising sub-model or not in later exploration.
As a result, DCC does not suffer from the reliance on forward sampling as per RMH to discover new SLPs.
Moreover, our scheme inherits the hill-climbing behavior of MCMC to discover SLP in a more efficient way:
it can find the sub-models of high posterior mass even under an extremely small prior probability,
as will be shown in \S\ref{sec:gmm-SLP-discovery}. 
Note that there are some subtleties to maintain the efficiency as the number of possible SLPs grow large, see Appendix~\ref{sec:app:disc} for further details.

\vspace{-2pt}
\subsection{Allocating Resources Between SLPs}
\label{sec:method-PPS-conquer-resource}

At each iteration we must choose an SLP from those discovered to perform local inference on. 
Though valid, it is not wise to split our computational resources evenly among all SLPs; it is more important to ensure we have accurate estimates for SLPs with large $Z_k$.
Essentially, we have a multi-armed bandit problem where we wish to develop a strategy of choosing SLPs that will lead to the lowest error in our \emph{overall} final estimate $\hat{\pi}$.
Though it might seem that this is a problem that DCC has introduced, it is actually an inherent underlying problem that must always be solved for models with stochastic support; DCC is simply making the problem explicit.
Namely, we do not know upfront which SLPs have significant mass and so any inference method must deal with the computational trade--off involved in figuring this out.
Conventional approaches do this in an implicit, and typically highly inefficient, manner.
For example, MCMC relies on the relatively frequency of individual transitions between SLPs to allocate resources, which will generally be extremely inefficient for finite budgets.

To address this, we introduce a resource allocation scheme based on an upper confidence bounding~(UCB) approach developed in~\citet{rainforth2018inference}.  
Specifically, we use the existing SLPs estimates to construct a utility function that conveys the relative merit of refining the estimates for each SLP, balancing the need for \emph{\textbf{exploitation}}, that is improving the estimates for SLPs currently believed to have large relative $Z_k$, 
and \emph{\textbf{exploration}}, that is improving our estimates for SLPs where our uncertainty in $Z_k$ is large.
At each iteration, we then update the estimate for the SLP which has the largest utility, defined as
\begin{align*}
\vspace{-5pt}
U_k := \frac{1}{S_k} \left(
\frac{(1-\delta)\hat{\tau}_k}{\max_k \{\hat{\tau}_k\}} 
+ \frac{\delta \hat{p}_k}{\max_k \{\hat{p}_k\}}  %
+  \frac{\beta \log \sum_{k} S_k}{\sqrt{S_k}}
\right)
\vspace{-5pt}
\end{align*}
where 
$S_k$ is the number of times we have previously performed local inference on $A_k$; $\hat{\tau}_k$ is the current estimate of the ``reward'' of $A_k$, incorporating both how much mass the SLP contains and how efficient our estimates are for it; $\hat p_k $ is a targeted exploration term that helps identify promising SLPs that we are yet to establish good estimates for; $0\le\delta\le1$ is a hyperparameter controlling the trade--off between these terms; and $\beta>0$ is the standard optimism boost hyper-parameter.
For more details see Appendix~\ref{sec:supp-resource-alloc}.

\section{Experiments}

\subsection{Gaussian Mixture Model (GMM)}
\label{sec:gmm}

We now further investigate the GMM example with an unknown number of clusters introduced in \S\ref{sec:inf-eng}.
Its program code written in Anglican is provided in Appendix~\ref{sec:supp-gmm}.
We compare the performance of DCC against five baselines: 
importance sampling (from prior)~(IS),
RMH~\cite{le2015rmh},
Particle Gibbs~(PG)~\cite{andrieu2010particle},
interacting Particle MCMC~(IPMCMC)~\cite{rainforth2016interacting},
and Black-box Variational Inference~(BBVI)~\cite{paige2016automatic},
taking the same computational budget of $10^6$ total samples for each. 

We first examine the convergence of the overall marginal likelihood estimate $\hat Z$.
Here IS is the only baseline which can be used directly, but we also consider
drawing importance samples centered around the RMH chain in a manner akin to PI-MAIS.
Figure~\ref{fig:GMM} [Left] shows that DCC outperforms both by many orders of magnitude. 
The sudden drop in the error for DCC is because the dominant sub-model with $K=5$ is typically discovered after taking around $10^4$ samples. 
Further investigation of that SLP allows DCC to improve the accuracy of the estimate.
DCC has visited $23$ to $27$ sub-models (out of \emph{infinitely} many) among all $15$ runs.

We next examine the posterior distribution of $K$ and report the estimates of $p(K=5\,|\,y_{1:N_y})$ in Figure~\ref{fig:GMM} [Right].
We see that all methods other than DCC struggle.
Here , the accuracy of the posterior of $K$ reflects the accuracy in estimating the relative masses of the different SLPs, i.e.~$Z_k$, explicitly or implicitly.
The dimension of this model varies between one and infinity and the posterior mass is concentrated in a small sub-region (${K=5}$) with small prior mass. 
It is therefore challenging for the baselines to either to learn each marginal likelihood simultaneously (eg.~in IS)
or to estimate the relative masses implicitly through transitions between configurations using an MCMC sampler.
By breaking down the model into sub-problems, DCC is able to overcome these challenges and provide superior posterior estimator for the overall model.

\vspace{-4pt}
\subsection{GMM with Misspecified Prior}
\label{sec:gmm-SLP-discovery}

To further test the capability of each method to discover SLPs---and to examine the MCMC-esque behavior for DCC in SLP space in particular---we adjust the GMM example above slightly so that $K$ now has a, high misspecified, prior of $\mathrm{Poisson}(90){+}1$, keeping everything else the same.
The dominant SLP is still ${K=5}$ (with around $0.9976$ posterior mass), but his now has an extremely low prior probability (around $10^{-14}$).
Consequently, finding this dominant SLP is only practically possible if the algorithm exhibits an effective hill climbing behavior in SLP space.
We only now compare DCC to RMH on the basis that: a) none of the baselines could deal with simpler case before, such that they will inevitably not be able to deal with this harder problem; and b) RMH is the only baseline where one might expect to see some hill climbing behavior in SLP space.

Figure~\ref{fig:gmm-poi90-visit} shows the trace plot for the SLP visit history (i.e.~sampled K at each iteration) of each method.
As we can see in the bottom plot,
DCC starts from the SLPs of $K$ around $90$, influenced by the prior, and gradually discovers smaller $K$s with higher posterior mass, also exploring large values of $K$ as well.
This implies a MCMC-esque hill-climbing behavior guided by our SLP discovery scheme.
Moreover, the trace plot also demonstrates the resource allocation within DCC where it gradually spends more computation for SLPs with lower $K$s while still maintaining a degree of exploration.
Both factors are essential for the resulting accurate posterior approximation shown in Figure~\ref{fig:gmm-poi90-p-k}.

By comparison, RMH gets stuck in its initialized SLP (Figure~\ref{fig:gmm-poi90-visit} top):
for the one run shown it only makes one successful transition to another SLP and never gets anywhere close to region of SLPs with significant mass.
Equivalent behavior was experience in all the other runs (not shown).
As a result, it does not produce a reasonable posterior estimate as shown Figure~\ref{fig:gmm-poi90-p-k}; in fact, it always returns an estimate of exactly $0$ as it never discovers this SLP.
It is worth noting that the local mixing of RMH between SLPs here is even worse than in the previous example.
This is because the larger $K$ at which the sampler is initialized induces a higher dimensional space on program draws, i.e.~$\mu_{1:K}$.
This is catastrophic for RMH because it is effectively importance sampling when transitioning between SLPs and thus suffers acutely from the curse of dimensionality.
DCC, meanwhile, gracefully deals with this because of its ability to remember SLPs that are proposed but not accepted and then subsequently perform effective localized inference that exploits hill--climbing effects in the space of the draws of that SLP.

\begin{figure}[!t]
	\centering
	\vspace{-5pt}
	\hspace{-11pt}
	\subfigure[SLP visit order]
	{\label{fig:gmm-poi90-visit}\includegraphics[width=0.6\linewidth, height=0.5\linewidth]{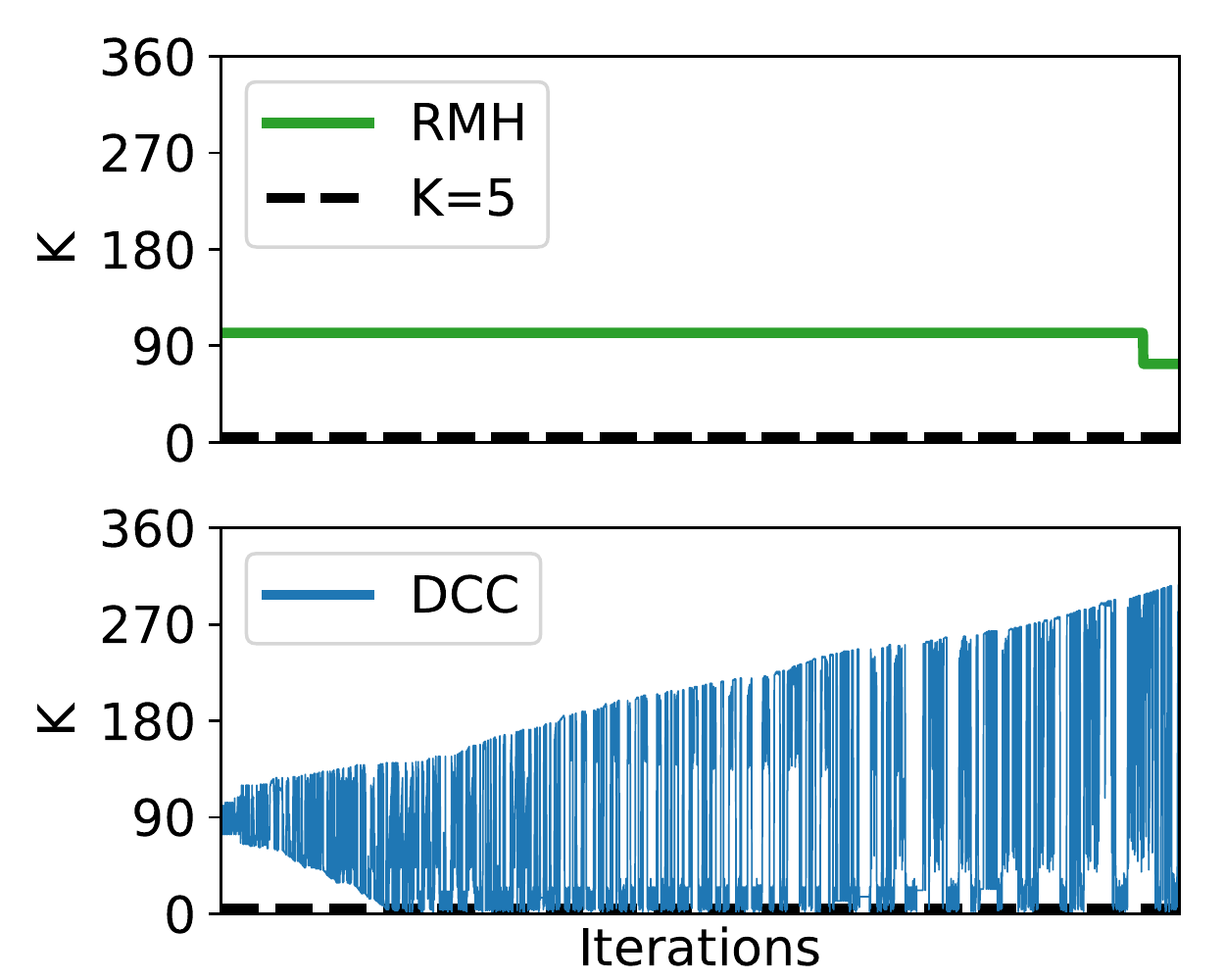}}%
	~~~~~~~~~~~
	\hspace{-20pt}
	\subfigure[$\hat{p}(K=5\,|\,y_{1:N_y})$]
	{\label{fig:gmm-poi90-p-k}\includegraphics[width=0.4\linewidth,height=0.51\linewidth]{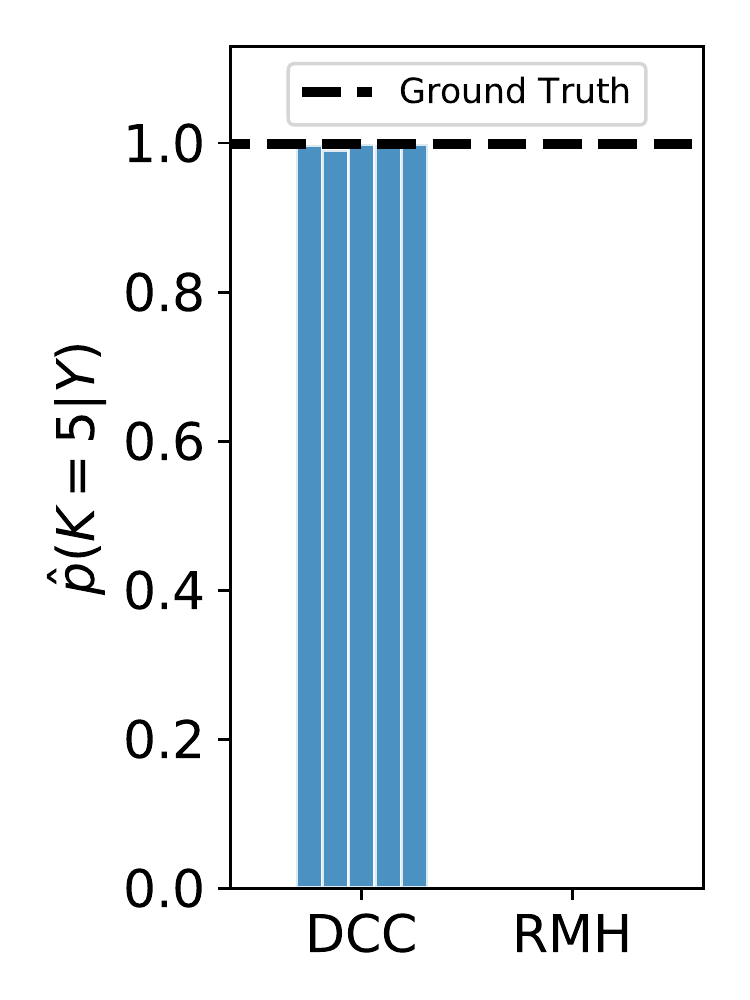}}%
	\vspace{-5pt}
	\caption{Comparison of DCC to RMH on GMM with misspecified prior. [Left] visit order of SLPs (i.e.~sampled $K$ at each iteration) for single run.  [Right] final posterior estimates for $5$ different runs.}
	\label{fig:gmm-poi90}
	\vspace{-15pt}
\end{figure}

\begin{figure*}[t!]
	\centering
	\subfigure[DCC~(ours)]
	{\centering\label{fig:PCFG-DCC}\frame{\includegraphics[width=0.32\linewidth, height=0.18\linewidth,trim={2.7cm 2.5cm 2cm 2cm},clip]{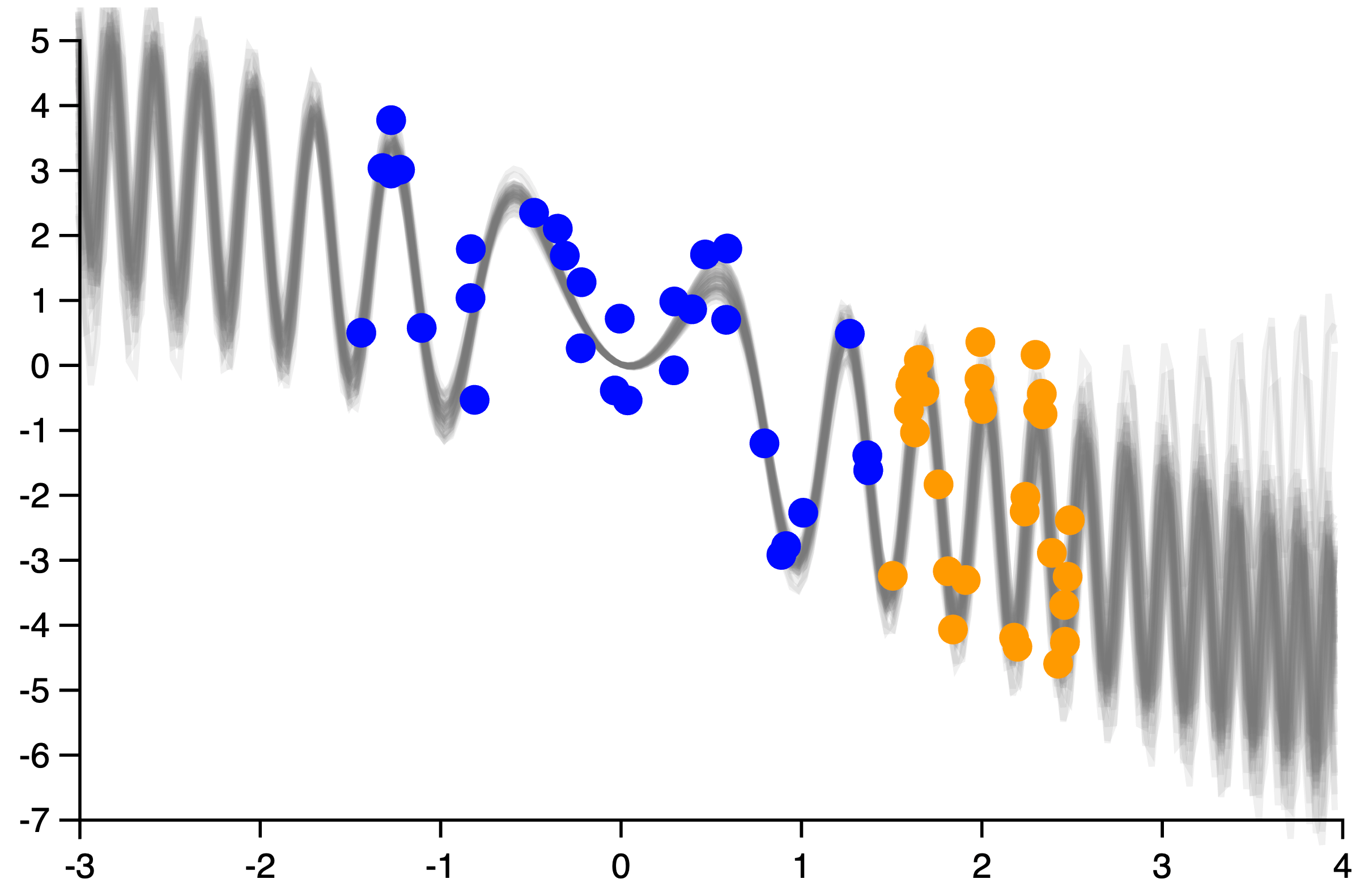}}}%
	\subfigure[IS]
	{~~~\,\centering\label{fig:PCFG-IS}\frame{\includegraphics[width=0.32\linewidth, height=0.18\linewidth,trim={2.7cm 2.5cm 2cm 2cm},clip]{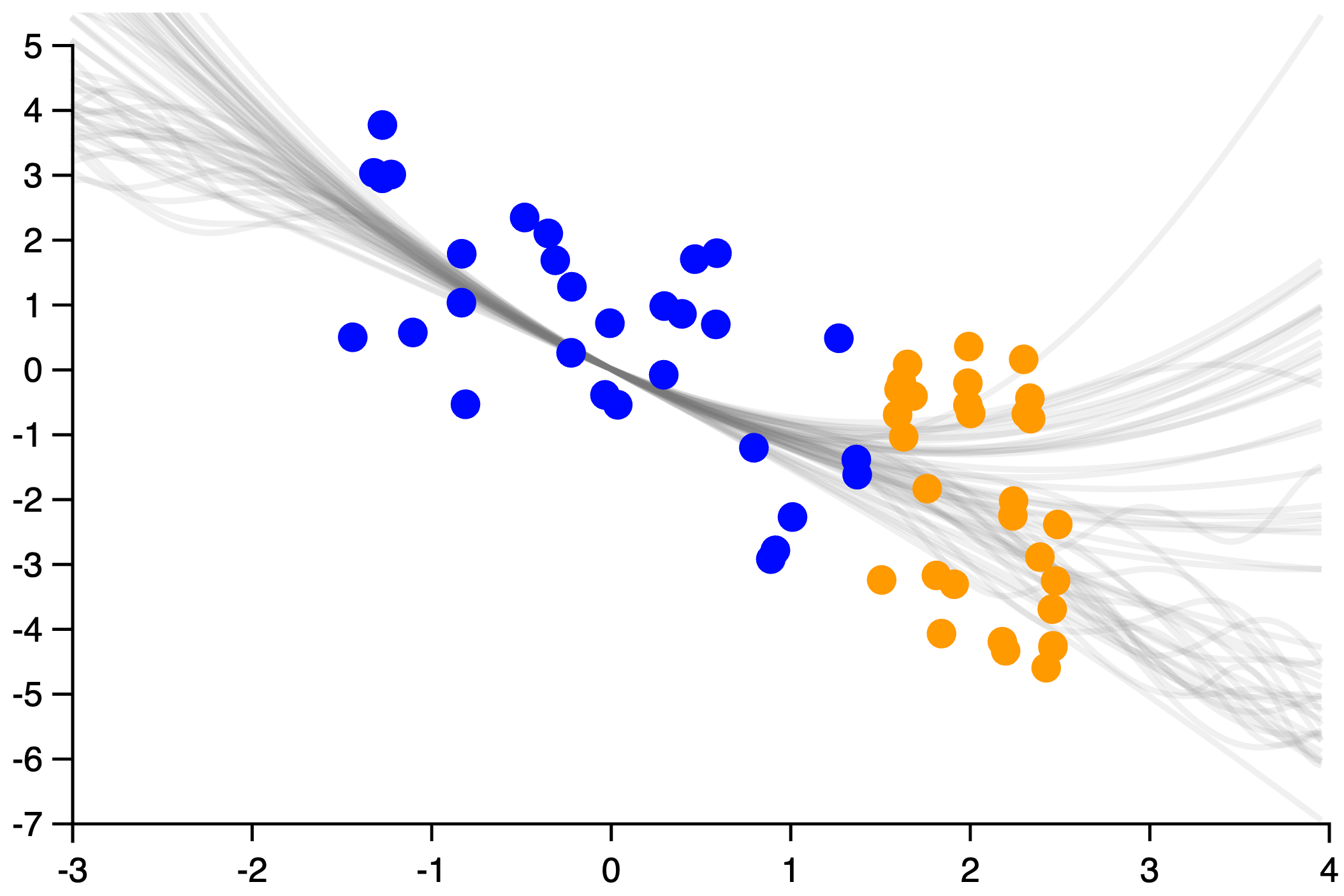}}}%
	\subfigure[RMH]
	{~~~\,\centering\label{fig:PCFG-RMH}\frame{\includegraphics[width=0.32\linewidth, height=0.18\linewidth,trim={2.7cm 2.5cm 2cm 2cm},clip]{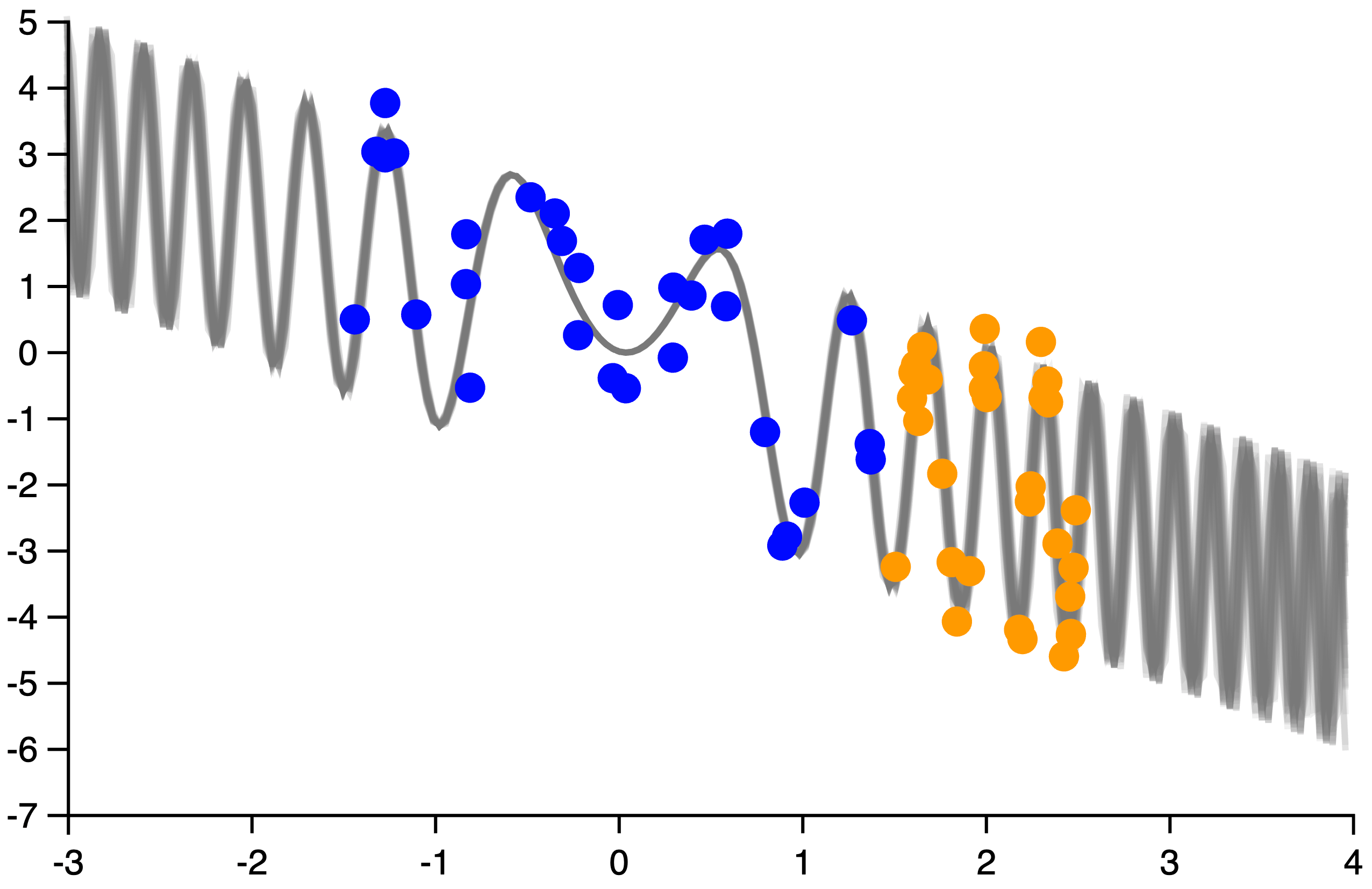}}}
	\vspace{-8pt}
	\caption{
		Posterior distribution $p(\Theta|D)$ estimated by DCC, IS, and RMH under the same computation.
		Blue points represent the observed data $D$ and orange ones the test data $D'$. 
		Grey lines are the posterior samples of the functions from the run with the highest LPPD among 15 independent runs of the three algorithms.
		\vspace{-10pt}
	}
	\label{fig:PCFG-qualitative}
\end{figure*}

\begin{table*}[!t]
	\centering
	\caption{Mean and one standard derivation of the LPPD over 15 independent runs.}
	\vspace{2pt}
		\small
		\begin{tabular}{|c|c|c|c|c|c|}
			\hline
			& \textbf{DCC~(ours)}                & IS                 & RMH           & PG& IPMCMC       \\ \hline
			LPPD & \textbf{-28.560 $\pm$ 0.41} & -73.180 $\pm$ 1.08 & -32.693 $\pm$ 8.51 & -200.824 $\pm$ 126.63 &-70.580 $\pm$ 4.71 \\ \hline
	\end{tabular}
	\label{tab:PCFG}
	\vspace{-5pt}
\end{table*}

\subsection{Function Induction}
\label{sec:pcfg}

Function induction is an important task for automated machine learning %
\cite{duvenaud2013structure, kusner2017grammar}. In PPSs, it is typically tackled  using a probabilistic context free grammar~(PCFG)~\cite{manning1999foundations}. 
Here we consider such a model where we specify the structure of a candidate function using a PCFG and a distribution over the function parameters, and estimate the posterior of both for given data.
Our PCFG consists of four production rules with fixed probabilities: $e \to \{x \,\mid\, x^2 \,\mid\, \sin(a*e) \,\mid\, a*e + b*e \}$, where $x$ and $x^2$ are terminal symbols,
$a$ and $b$ are unknown coefficient parameters, and
$e$ is a non-terminal symbol. %
The model also has prior distributions over each coefficient parameters. See Appendix~\ref{sec:supp-pcfg} for details.

To generate a function from this model, we must sample both a PCFG rollout and the corresponding parameters.
Let $\Theta$ be the collection of all the latent variables used in this generative process. 
That is, $\Theta$ consists of the sequence of the discrete variables recording the choices of the grammar rules
and all coefficients in the sampled structure.
Conditioned on the training data $D$, we want to infer the posterior distribution $p(\Theta|D)$, and calculate the posterior predictive distribution $p(D'|D)$ for test data $D' = \{x_n, y_n\}_{n=1}^{N}$.

In our experiment, we control the number of sub-models by requiring that the model use the PCFG in a restricted way: a sampled function structure should have depth at most $3$ and cannot use the plus rule consecutively. We generate a synthetic dataset of $30$ training data points %
from the function $f(x) = -x + 2\sin (5x^2)$
and compare the performance of DCC to our baselines on estimating the posterior distribution and the posterior predictive under the same computational budget of $10^6$ samples and $15$ independent runs.
BBVI is omitted from this experiment due to it failing to run at all.

Figure~\ref{fig:PCFG-qualitative} shows the posterior samples generated by DCC, IS, and RMH for one run, with the training data $D$ marked blue and the test data $D'$ in orange. The DCC samples capture the periodicity of the data and provides accurate extrapolation, while retaining an appropriate degree of uncertainty. This indicates good inference
results on both the structure of a function and the coefficients.
Though RMH does find some good functions, it becomes stuck in a particular mode and does not fully capture the uncertainty in the model, leading to poor predictive performance.

Table~\ref{tab:PCFG} shows the test log posterior predictive density (LPPD),  %
$\sum_{n=1}^{N} \log  \int_{\Theta} p(y_n | x_n, \Theta) p(\Theta | D) d\Theta$,
of all approaches.
DCC substantially outperforms all the baselines both in terms of predictive accuracy and stability.
IS, PG, and IPMCMC all produced very poor posterior approximations leading to very low LPPDs.
RMH had an LPPD that is closer to DCC, but which is still substantially inferior.

A further issue with RMH was its high variance of the LPPD. 
This is caused by this model being multi-modal and RMH struggling to move: it gets stuck in a single SLP and fails to capture the uncertainty.
Explicitly, $4$ sub-models (out of $26$) contain most of the probability mass.
Two of them are functions of the form used to generate the data, $f(x) = a_1x + a_2 \sin(a_3x^2)$, modulo symmetry of the $+$ operator.  
The other two have the form $f(x) = a_1 \sin (a_2 x) + a_3 \sin(a_4x^2) $, which can also match the training data well in the region ($-1.5, 1.5$) as $a_1\sin(a_2 x) \approx a_1 a_2 x$ for small values of $a_2 x$. 
Note that the local distributions are also multi-modal due to various symmetries, for example $a_1 \sin(a_2x^2)$ and $-a_1 \sin(-a_2x^2)$, meaning the local inference task is non-trivial even in low dimensions. 

\begin{figure}
	\hspace{-8pt}
	\begin{minipage}[c]{0.285\textwidth}
		\includegraphics[width=\textwidth]{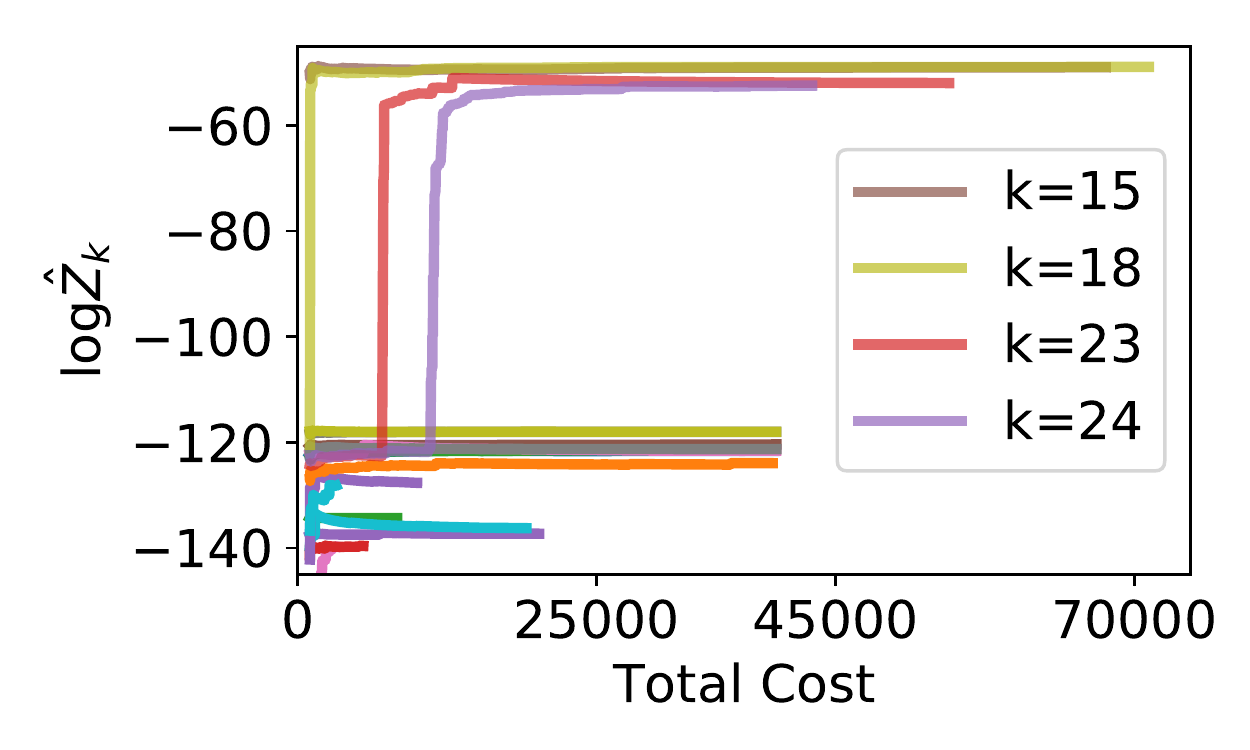}
	\end{minipage} ~~
	\begin{minipage}[c]{0.19\textwidth}
		\caption{Convergence of DCC's $\log Z_k$ estimate for each SLP and corresponding total amount of resources spent.\vspace{10pt}}
 	\label{fig:PCFG-logZ-Ks}
	\end{minipage}
\vspace{-20pt}
\end{figure}
To test the effectiveness of the resource allocation strategy, 
we further investigate the 
computational resources spent for each SLP
by looking at the convergence of the local marginal likelihood estimates $\hat{Z}_k$. %
In Figure~\ref{fig:PCFG-logZ-Ks},  
the sub-models $15$ and $18$ correspond to the form $f(x) = a_1x + a_2 \sin(a_3x^2)$  and its mirror, which contain the most posterior mass. 
The sub-models $23$ and $24$ correspond to $f(x) = a_1 \sin (a_2 x) + a_3 \sin(a_4x^2) $ and are the second largest modes. %
Figure~\ref{fig:PCFG-logZ-Ks} implies that DCC indeed spends 
more computational resource on sub-models with high probability mass (as signified by the higher final total cost), 
while also exploring the other sub-models occasionally. 

\vspace{3pt}
\section{Conclusions}
In this paper, we have proposed \emph{Divide, Conquer, and Combine (DCC)}, a new inference strategy
for probabilistic programs with stochastic support.
We have shown that, by breaking down the overall inference problem into a number of separate inferences of sub-programs with fixed support, the DCC framework can provide substantial performance improvements over existing approaches which directly target the full program.
To realize this potential, we have shown how to implement a particular instance of DCC as an automated engine in the PPS Anglican and shown that this outperforms existing baselines on three example problems.

\clearpage
\section*{Acknowledgements}
YZ is sponsored by China Scholarship Council (CSC).
HY was supported by the Engineering Research Center Program through the National Research Foundation of Korea (NRF) funded by the Korean Government MSIT (NRF-2018R1A5A1059921), and also by
Next-Generation Information Computing Development
Program through the National Research Foundation
of Korea (NRF) funded by the Ministry of Science,
ICT (2017M3C4A7068177). YWT’s and TR’s research
leading to these results has received funding from
the European Research Council under the European
Union’s Seventh Framework Programme (FP7/2007-
2013)/ ERC grant agreement no. 617071. TR was also
supported in part by Junior Research Fellowship from
Christ Church, University of Oxford and and in part
by EPSRC funding under grant EP/P026753/1.

\bibliography{refs}
\bibliographystyle{icml2020}

\clearpage
\onecolumn
\appendix

	\thispagestyle{empty} 
	\rule{\textwidth}{1pt}
	\vspace{-8pt}
	\begin{center}
		\textbf{ \Large Appendix for Divide, Conquer, and Combine: a New Inference Strategy \\ 
			for Probabilistic Programs with Stochastic Support}
	\end{center}\vspace{-8pt}
	\rule{\textwidth}{1pt}
	\icmltitlerunning{Appendix: Divide, Conquer, and Combine}

\section{Existing PPSs and Inference Engines for Probabilistic Programs with Stochastic Support}
\label{sec:supp-discussion-PPSs}

\begin{table}[!h]
	\centering
	\caption{List of popular universal PPSs and their supported inference algorithms for models with stochastic support.\footnotemark}
	\vspace{10pt}
	\label{tab:existing-PPSs}
	\begin{tabular*}{\textwidth}{c|ccccccc}
		\thickhline
		& Rejection    & IS           & SMC          & PMCMC 
		&  VI 
		& \begin{tabular}[c]{@{}c@{}}Customized \\proposal MCMC \end{tabular} 
		& \begin{tabular}[c]{@{}c@{}}Automated \\proposal MCMC\end{tabular} %
		\\ \thickhline
		\begin{tabular}[c]{@{}c@{}}Venture\\ \cite{mansinghka2014venture} \end{tabular}       
		& $\checkmark$ &              &    $\checkmark$     &     $\checkmark$                              &    $\checkmark$   &                                                             &      $\checkmark$         \\ \hline
		 \begin{tabular}[c]{@{}c@{}}WebPPL\\ \cite{goodman2014design} \end{tabular}  
		& $\checkmark$ &              & $\checkmark$ & $\checkmark$         &  &                        &   $\checkmark$           \\ \hline
		\begin{tabular}[c]{@{}c@{}}MonadBayes\\ \cite{scibior2018functional} \end{tabular}        
		&              &              & $\checkmark$         & $\checkmark$                                                           &              &                                                                       &              \\ \hline
		\begin{tabular}[c]{@{}c@{}}Anglican\\ \cite{wood2014new} \end{tabular}   
		&              & $\checkmark$ & $\checkmark$ & $\checkmark$                                                           & $\checkmark$ &                                                                       & $\checkmark$ \\ \hline
		 \begin{tabular}[c]{@{}c@{}}Turing.jI\\ \cite{ge2018turing} \end{tabular}   
		&              & $\checkmark$ & $\checkmark$ & $\checkmark$                                                           &              &         $\checkmark$     &   $\checkmark$           \\ \hline
       \begin{tabular}[c]{@{}c@{}}Pyro\\ \cite{bingham2018pyro} \end{tabular}   
		&              & $\checkmark$ & $\checkmark$ &                                                                        &       $\checkmark$       &        $\checkmark$      &  \\ \hline
		\begin{tabular}[c]{@{}c@{}}Gen\\ \cite{cusumano2019gen} \end{tabular}             
		&              & $\checkmark$ & $\checkmark$ &                          &      $\checkmark$        & $\checkmark$                                                          &      $\checkmark$      \\ \hline
		\begin{tabular}[c]{@{}c@{}}Hakaru\\ \cite{narayanan2016probabilistic} \end{tabular}        
		&              &   $\checkmark$     &              &                                                                        &              & $\checkmark$                                                          &              \\ \hline
		\begin{tabular}[c]{@{}c@{}}Stochaskell \\ \cite{roberts2019reversible} \end{tabular}        
		&              &              &              &                                                                        &              & $\checkmark$                                                          &              \\ \thickhline
	\end{tabular*}
\vspace{10pt}
\end{table}
\footnotetext{All information are taken from the published paper or the online manual.}

In Table~\ref{tab:existing-PPSs} above, we have listed which inference engines from the five categories in \S\ref{sec:inf-eng} are supported by each of the most popular universal PPSs.
The fundamental difficulty in performing inference on probabilistic models with stochastic support is that
the posterior mass of such models is usually concentrated in one or many separated sub-regions which are non-trivial to be fully discovered, especially in a high dimensional space.  
Moreover, even if a variable exists in different variable configurations, the posterior might shift substantially like the mean $\mu_1$ for the GMM shown in the Figure~\ref{fig:GMM-demo-2}, which complicates the design for a ``proper'' proposal distribution, let alone automate this procedure in PPS. 
We now have a more detailed look at each category with the related PPSs and uncover the reasons why these engines are not suitable for probabilistic programs with stochastic support. 

\paragraph{Importance/rejection sampling}  
Basic inference schemes, such as importance sampling~(IS) and rejection sampling~(RS), are commonly supported by many PPSs due to their generality and simplicity. See Table~\ref{tab:existing-PPSs}.
They can be directly applied to a model with stochastic support, but their performance deteriorates rapidly (typically exponentially) as the dimension of the model increases; they suffer acutely from the curse of dimensionality.
Furthermore, their validity relies on access to a valid proposal.  
This can often be easily ensured, e.g.~by sampling from the prior, but doing this while also ensuring the proposal is efficient can be very challenging; the prior is only a practically viable proposal for very simple problems.
Though schemes for developing more powerful proposals in amortized inference settings have been developed~\citep{le2016inference,ritchie2016deep}, these are not appropriate for conventional inference problems.

\paragraph{Particle based methods}
Particle based inference methods, such as Sequential Monte Carlo~(SMC)~\cite{doucet2001introduction}, can offer improvements for models with natural sequential structure~\cite{wood2014new,rainforth2016interacting}.
More explicitly, SMC improves IS when there exist interleaved observe statements in the model program.
However, it similarly rapidly succumbs to the curse of dimensionality in the more general case where this does not occur.
Such methods also cannot be used at all if the number of observe statements is not fixed, 
which can be common for stochastic support problems.

One might then tend to more sophisticated methods such as particle MCMC~(PMCMC) methods~\cite{andrieu2010particle}
but unfortunately these suffer from the same underlying issues.
The basic setups of PMCMC include the Particle Independent Metropolis Hasting~(PIMH) (in Anglican and Turing.jI) and Particle Gibbs~(PG) (in Anglican and Turing.jI~\cite{ge2018turing}), where one uses, respectively, \emph{independent} and \emph{conditional} SMC sweeps as the proposal within the MCMC sampler. 
These building-block sweeps suffer from exactly the same issue as conventional SMC, and thus offer no advantage over simple importance sampling without interleaved observe statements.

These advanced variants do though allow one to treat global parameters separately:
they update the global parameters using a Metropolis--within--Gibbs (MwG) step and then update the rest latent variables using a (conditional) SMC sweep with those global parameter values.
When combined with PIMH, this leads to Particle Marginal Metropolis Hasting~(PMMH) algorithm (available in WebPPL~\cite{goodman2014design}, Turing.jI and MonadBayes~\cite{scibior2018functional}). 
It already constitutes a valid step for PG.
Unfortunately, in both cases this MwG suffers from exactly the same issues as those already discussed for LMH.
As such, these approaches again offer little advantage over importance sampling without sequential problem structure.

\paragraph{Variational inference}
Following the discussion in \S~\ref{sec:inf-eng},
three universal PPSs under our survey that allow Variational Inference~(VI) in stochastic support settings are Pyro~\cite{bingham2018pyro}, Gen~\cite{cusumano2019gen} and Anglican~\cite{wood2014new} (note many others allow VI for statistic support, but cannot be used when the support varies).
Pyro supports Stochastic Variational Inference~(SVI)~\cite{hoffman2013stochastic, wingate2013automated, kucukelbir2017automatic} and allows an auto-generated guide function~(AutoGuide), i.e. the variational proposal, or a user-specified guide.
However, AutoGuide only works for some basic probabilistic programs, which we cannot directly apply for the GMM example in \S~\ref{sec:inf-eng}.
With the customized guide, one cannot deal with the exact same model because of the need to upper-bound the number of the variational parameters according to the tutorial\footnote{\url{https://pyro.ai/examples/dirichlet_process_mixture.html}}.

Both Gen and Anglican support the Black box Variational Inference~(BBVI)~\cite{ranganath2014black}.
We have tested the Anglican's BBVI with the GMM example in \S~\ref{sec:inf-eng}, but it does not provide very accurate estimates as can be seen in the Figure~\ref{fig:GMM-demo}.

One key challenge for implementing VI dealing with models with stochastic support is that one might still have never-before-seen variables after finite number of training steps.
Moreover, it is usually very challenging in stochastic support settings to ensure that the variational family is defined in a manner that ensures the KL is well defined.  
For example, if using $\text{KL}(q||p)$ (for proposal $q$ and target $p$), then the KL will be infinite if $q$ places support anywhere $p$ does not.  
Controlling this with static support is usually not especially problematic, but in stochastic support settings it can become far more challenging.

Overcoming these complications is beyond the scope of our paper but could be a potential interesting direction for future work.

\paragraph{MCMC with customized proposal}
To perform Markov chain Monte Carlo~(MCMC) methods~\cite{metropolis1949monte} on the models with stochastic support, one needs to construct the transitional kernel such that the sampler can switch between configurations. 
A number of PPSs such as Hakaru, Turing.jI, Pyro and Gen allow the user to customize the kernel for Metropolis Hastings (a.k.a programmable kernel) whereas Stohaskell explicitly supports reversible jump Markov chain Monte Carlo~(RJMCMC)~\cite{green1995reversible, green2003trans} methods. 
However, their application is fundamentally challenging as we have discussed in \S~\ref{sec:inf-eng} due to the difficulty in designing proposals which can transition efficiently as well as the posterior shift on the variable under different configuration.

\paragraph{MCMC with automated proposal}
One MCMC method that can be fully automated for PPSs is the Single-site Metropolis Hastings or the Lightweight Metropolis Hastings algorithm (LMH) of~\cite{wingate2011lightweight}
and its extensions~\cite{yang2014generating,TolpinMPW15,le2015rmh, RitchieSG16}, for which implementations are provided in a number of systems such as
Venture~\cite{mansinghka2014venture}, %
WebPPL %
and Anglican. %
In particular, Anglican supports LMH and its variants, Random-walk lightweight Metropolis Hastings (RMH)~\cite{le2015rmh}, which uses a mixture of prior and local proposal to update the selected entry variable $x_i$ along the trace $x_{1:n_x}$.

Many shortcomings of LMH and RMH have been discussed in \S\ref{sec:inf-eng} and we reiterate a few points here.
Though widely applicable, LMH relies on proposing from the prior whenever the configuration changes for the downstream variables.
This inevitably forms a highly inefficient proposal (akin to importance sampling from the prior), such that although samples from different configurations might be proposed frequently, these samples might not be ``good'' enough to be accepted. 
This usually causes LMH get stuck in one sub-mode and struggles to switch to other configurations. 
For example, we have observed this behavior in the GMM example in Figure~\ref{fig:GMM}. 
As a result, LMH typically performs very poorly for programs with stochastic support, particularly in high dimensions.

Note that LMH/RMH could have good mixing rates for many models with fixed support in general akin to standard Metropolis-within-Gibbs methods.
This is because when $x_i$ is updated, the rest of the random variables can still be re-used, which ensures a high acceptance rate. 
That is also why we can still use LMH/RMH as the local inference algorithm within the DCC framework for a fixed SLP to establish substantial empirical improvements.

\section{Detailed Algorithm Block for DCC in Anglican}
\label{sec:app:alg-block}

\begin{algorithm}[H]
	\caption{An Implementation of DCC in Anglican \label{alg:div-con-anglican}}
	\begin{algorithmic}[1]
		\floatname{algorithm}{Procedure}
		\renewcommand{\algorithmicrequire}{\textbf{Input:}}
		\renewcommand{\algorithmicensure}{\textbf{Output:}}	
		\renewcommand{\COMMENT}[1]{\hfill$\triangleright$ {#1}}
		
		\REQUIRE
		Program $\boldsymbol{\prog}$, number of iterations $T$, inference hyper-parameters $\Phi $ (e.g. number of initial iterations $T_0$, times proposed threshold $C_0$), 
		\ENSURE 
		Posterior approximation $\hat{\pi}$ and marginal likelihood estimate $\hat{Z}$
		
		\STATE Execute $\boldsymbol{\prog}$ forward multiple times (i.e.~ignore observes) to obtain an initial set of discovered SLPs $A^{total}$ 
		\FOR{$t=1,\dots,T$}
		\STATE Select the model(s) $k'$ in $A^{total}$ whose proposed times $C_{k'} \geq C_0$
		and add them into $A^{active}$
		\IF{exist any new models}
		\STATE Initialize the inference with $N$ parallel MCMC chains
		\STATE Perform $T_{init}$ optimization step for each chain by running a ``greedy'' RMH locally and only accepting samples with higher $\hat \gamma_k$ to boost burning-in; 
					store only the last MCMC samples as the initialization for each chain
		\STATE Draw $M$ importance samples using each previous MCMC samples as proposal to generate initial estimate for $\hat Z_k$
		\ENDIF
		\item[\qquad \; \textbf{Step 1: select a model $\mathbf{A_{k^*}}$}]
		\STATE 
		Choose a sub-model with index $k^*$ with the maximum utility value as per Equation~\ref{eq:ucb}
		\item[\qquad \; \textbf{Step 2: perform local inference on $\mathbf{A_{k^*}}$}]
		\STATE Perform one or more RMH step locally for all $N$ chains of model $K^*$ to update $\hat{\pi}_{k^*}$
		\STATE Draw $M$ importance samples using each previous MCMC samples as proposal to update $\hat{Z}_{k^*}$
		
		\item[\qquad \; \textbf{Step 3: explore new models} (optional)]
		\STATE Perform one RMH step using a global proposal for all $N$ chains to discover more SLPs $A_{k''}$
		\STATE Add $ A_{k''}$ to $A^{total}$ if $ A_{k''} \notin A^{total}$;  increment the $C_{k''}$
		\ENDFOR
		
		\STATE Combine local approximations into overall posterior estimate $\hat{\pi}$ and overall marginal likelihood estimate $\hat{Z}$ as per~\eqref{eq:comb_est}
	\end{algorithmic}
\end{algorithm}

\section{Details for Proof of Theoretical Correctness}
\label{sec:app:theory}

\setlength{\abovedisplayskip}{2pt}
\setlength{\belowdisplayskip}{3pt}
\setlength{\abovedisplayshortskip}{2pt}
\setlength{\belowdisplayshortskip}{3pt}

In this section, we provide a more formal demonstration of the consistency of the DCC algorithm.
We start by explaining the required assumptions, before going on to the main theoretical result.

More formally, our first assumption, which may initially seem highly restrictive but turns out to be innocuous, is that we only split our program into a finite number of sub-programs:
\begin{assumption}
	\label{assumption1}
	The total number of sub-programs $K$ is finite.
\end{assumption}
We note that this assumption is not itself used as part of the consistency proof, but is a precursor to Assumptions~\ref{assumption4} and~\ref{assumption5} being satisfiable.
We can always ensure the assumption is satisfied even if the number of SLPs is infinite; we just need to be careful about exactly how we specify a sub-program.
Namely, we can introduce a single sub-program that combines all the SLPs whose path is longer than $n_{\text{thresh}}$, i.e.~those which invoke $n_x > n_{\text{thresh}}$ sample statements, and then ensure that the local inference run for this specific sub-program is suitable for problems with stochastic support (e.g.~we could ensure we always use importance sampling from the prior for this particular sub-program).
If $n_{\text{thresh}}$ is then set to an extremely large value such that we can safely assume that the combined marginal probability of all these SLPs is negligible, this can now be done without making any notable adjustments to the practical behavior of the algorithm. 
In fact, we do not even envisage this being necessary for actual implementations of DCC: it is simply a practically inconsequential but theoretically useful adjustment of the DCC algorithm to simplify its proof of correctness.
Moreover, the fact that we will only ever have finite memory to store the SLPs means that practical implementations will generally have this property implicitly anyway.

For better notational consistency with the rest of the paper, we will use the slightly inexact convention of referring to each of these sub-programs as an SLP from now on, such that the $K^{\text{th}}$ ``SLP'' may actually correspond to a collection of SLPs whose path length is above the threshold if the true number of SLPs is infinite.

Our second and third assumptions simply state that our local estimators are consistent given sufficient computational budget:
\begin{assumption}
	\label{assumption2}
	For every SLP $k \in \{1,\dots,K\}$, we have a local density estimate $\hat{\pi}_k$ (taking the form on an empirical measure) which converges weakly to the corresponding conditional distribution of that SLP $\pi_k$ in limit of large allocated budget $S_k$, where $\pi_k(x) \propto \gamma(x) \iden[x\in \mathcal{X}_k]$, $\gamma(x)$ is the unnormalized distribution of the program, and $\mathcal{X}_k$ is the support corresponding to the SLP $k$.
\end{assumption}
\begin{assumption}
	\label{assumption3}
	For every SLP, we have a local marginal probability estimate $\hat{Z}_k$ which converges in probability to the corresponding to true marginal probability $Z_k = \int \gamma(x) \iden[x\in \mathcal{X}_k] dx$ in limit of large allocated budget $S_k$.
	We further assume that if $Z_k=0$, then $\hat{Z}_k$ also equals $0$ with probability $1$ (i.e.~we never predict non-zero marginal probability for SLPs that contain no mass).
\end{assumption}
The final part of this latter assumption, though slightly unusual, will be satisfied by all conventional estimators: it effectively states that we do not assign finite mass in our estimate to any SLP for which we are unable to find any valid traces with non-zero probability.
	
Our next assumption is that the SLP extraction strategy will uncover all $K$ SLPs in finite time.
\begin{assumption}
	\label{assumption4}
	Let $T_{\text{found}}$ denote the number of iterations the DCC approach takes to uncover all SLPs with $Z_k>0$.
	We assume that $T_{\text{found}}$ is almost surely finite, i.e.~$P(T_{\text{found}}<\infty) = 1$.
\end{assumption}
A sufficient, but not necessary, condition for this assumption to hold is to use a method for proposing SLPs that has a non-zero probability of proposing a new trace from the prior.
This condition is satisfied by LMH style proposals like the RMH proposal we adopt in practice.
We note that as with Assumption~\ref{assumption1}, this assumption is not itself used as part of the consistency proof, but is a precursor to Assumption~\ref{assumption5} (below) being satisfiable.

Our final assumption is that our resource allocation strategy asymptotically allocates a finite proportion of each of its resources to each SLP with non-zero marginal probability:
\begin{assumption}
	\label{assumption5}
	Let $T$ denote the number of DCC iterations and $S_k(T)$ the number of times that we have chosen to allocate resources to an SLP $k$ after $T$ iterations.
	We assume that there exists some $\epsilon>0$ such that
	\begin{align}
                \forall k \in \{1,\dots,K\}.\;
                Z_k>0 \implies \frac{S_k(T)}{T} > \epsilon.
	\end{align}
\end{assumption}

Given these assumptions, we are now ready to demonstrate the consistency of the DCC algorithm as follows:
\constistency*
\begin{proof}
By Assumption~\ref{assumption5}, we have that for all $k$ with $Z_k > 0$, $S_k\to\infty$ as $T\to\infty$. 
Using this along with Assumptions~\ref{assumption2} and~\ref{assumption3} gives us that, in the limit $T\to\infty$,
\begin{align}
\label{eq:zk_conv}
\hat{Z}_k &\pto Z_k \quad \forall k\in\{1,\dots,K\} \\
\label{eq:pik_conv}
\hat{\pi}_k (\cdot) &\dto \pi_k(\cdot) \quad \forall k\in\{1,\dots,K\} \text{ with } Z_k > 0
\end{align}
so that all our local estimates converge.

The result now follows through a combination of Equation~\ref{eq:comb_est}, linearity, and Slutsky's theorem.
Namely, let us consider an arbitrary bounded continuous function $f: \mathcal{X} \rightarrow \mathbb{R}$, for which we have
\begin{align*}
\int f(x)\hat \pi(dx) 
&= \frac{\int f(x) \sum_{k=1}^K \hat{Z}_k \hat \pi_k(dx) }{\sum_{k=1}^K \hat{Z}_k} \\
&= \frac{\sum_{k=1}^K \int f(x)  \hat{Z}_k \hat \pi_k(dx)}{\sum_{k=1}^K \hat{Z}_k}.
\end{align*}
Given Equations~\eqref{eq:zk_conv} and~\eqref{eq:pik_conv}, using Slutsky's theorem, we can conclude that as $T\to\infty$, the above integral converges to
\begin{align*}
\frac{\sum_{k=1}^K \int f(x) Z_k \pi_k(dx) }{\sum_{k=1}^K Z_k} &= \frac{ \int f(x) \sum_{k=1}^K Z_k \pi_k(dx)}{Z} \\
&= \frac{ \int f(x)\gamma(dx) }{Z}
\\
&= \mathbb{E}_{\pi(x)}[f(x)].
\end{align*}
We thus see that the estimate for the expectation of $f(x)$, which is  calculated using our empirical measure $\hat{\pi}\left(\cdot\right)$, converges to its true expectation under $\pi(\cdot)$.
As this holds for an arbitrary integrable $f(x)$, this ensures, by definition, that $\hat{\pi}\left(\cdot\right)$ converges in distribution to $\pi(\cdot)$, thereby giving the desired result.
\end{proof}

We finish by noting that our choices for the particular DCC implementation in Anglican straightforwardly ensure that these assumptions are satisfied (provided we take the aforementioned care around Assumption~\ref{assumption1}).
Namely:
\begin{itemize}
	\item Using RMH for the local inferences will provide $\hat{\pi}_k$ that satisfies Assumption~\ref{assumption2}.
	\item Using PI-MAIS will provide $\hat{Z}_k$ that satisfies Assumption~\ref{assumption3} provided we construct this with a valid proposal.
	\item The method for SLP extraction has a non-zero probability of discovering any of the SLPs with $Z_k>0$ at each iteration because it has a non-zero probability of proposing a new trace from prior, which then itself has a non-zero probability of proposing each possible SLP.
	\item The resource allocation strategy will eventually choose each of its possible actions with non-zero rewards (which in our case are all SLPs with $Z_k>0$) infinitely often, as was proven in~\cite{rainforth2018inference}.
\end{itemize}

\section{Additional Details on Local Estimators}
\label{sec:app:local-inf}

Recall that the goal for the local inference is to estimate the local target density $\pi_k(x)$ 
(where we only have access to $\gamma_{k}(x)$), 
and the local marginal likelihood $Z_k$ for a given SLP $A_k$.
Each SLP has a fixed support, i.e. a fixed configuration of the random variables, now where many of the complicated factors from varying support no longer apply. 

To estimate the local target density $\pi_k(x)$ for a given SLP $A_k$, DCC establishes a multiple-chain MCMC sampler in order to ensure a good performance in the setting with high dimensional and potentially multi-modal local densities.
Explicitly, we perform one RMH step in each chain for $N$ independent chains in total at each iteration.
Suppose the total iteration to run local inference in $A_k$ is $T_k$. 
With all the MCMC samples $(\hat x^{(k)}_{1:N, 1:T_k})$ within $A_k$,
we then have the estimator
\begin{align}
\label{eq:hat-pi-k}
\hat \pi_k(x) := \frac{1}{NT_k}\sum_{n=1}^{N}\sum_{t=1}^{T_k} \delta_{\hat x^{(k)}_{n, t}}(\cdot).
\end{align} 

As MCMC samplers do not directly provide an estimate for $Z_k$, we must introduce a further estimator that uses these samples to estimate it.
For this, we use PI-MAIS~\cite{martino2017layered}.
Though ostensibly an adaptive importance sampling algorithm, PI-MAIS~\cite{martino2017layered} is based around using the set of $N$ proposals each centered on the outputs of an MCMC chain. 
It can be used to generates marginal likelihood estimates from a set of MCMC chains, as we require.  

More precisely, given the series of previous generated MCMC samples, $\hat{x}^{(k)}_{1:N,1:T_k}$,  PI-MAIS introduces a mixture proposal distribution for each iteration of the chains 
by using the combination of separate proposals (e.g. a Gaussian) centered on each of those chains:
\begin{align}
q^{(k)}_{t}(\cdot | \hat{x}^{(k)}_{1:N,t}) := 
\frac{1}{N} \sum_{n=1}^{N} q^{(k)}_{n,t}(\cdot | \hat{x}^{(k)}_{n,t})
        \quad\text{ for } t \in \{1,2,\ldots,T_k\}.
\end{align}
This can then be used to produce an importance sampling estimate
for the target, with Rao-Blackwellization typically applied across the mixture components, 
such that $M$ samples, $(\tilde x^{(k)}_{n, t, m})^M_{m=1}$, are drawn separately from each $q^{(k)}_{n,t}$ with weights
\begin{align}
\tilde w^{(k)}_{n, t, m} := \frac{\gamma_k(\tilde x^{(k)}_{n, t, m})}{q^{(k)}_{t}( \tilde{x}^{(k)}_{n,t,m} |  \hat{x}^{(k)}_{n,t} )}
\,\text{ with }\, \tilde x^{(k)}_{n, t, m} \sim q^{(k)}_{n,t}(\cdot | \hat{x}^{(k)}_{n,t}),
        \qquad \text{for }\, m \in \{1,\ldots,M\} \,\text{ and }\, n \in \{1,\ldots,N\}. 
\end{align}
We then have the marginal likelihood estimate $\hat{Z}_k$ as
\begin{align}
\label{eq:hat-Z-k}
\hat Z_k := \frac{1}{NT_kM} \sum_{n=1}^{N}\sum_{t=1}^{T_k}\sum_{m=1}^{M} \tilde{w}^{(k)}_{n, t, m}.
\end{align}

An important difference for obtaining $\hat Z_k$ using an adaptive IS scheme with multiple MCMC chain as the proposal, compared to vanilla importance sampling (from the prior), 
is that MCMC chains form a much more efficient proposal than the prior as they gradually converge to the local target distribution $\pi_k(x)$.
These chains are running locally, i.e. restricted to the SLP $A_k$, which means that the issues of the LMH transitioning between SLPs as discussed in \S\ref{sec:inf-eng} no longer apply and local LMH could maintain a much higher mixing rate where it becomes the standard MwG sampler on a fixed set of variables. 
Furthermore, the benefit of having multiple chains is that they will approximate a multi-modal local density better. 
With $N$ chains, we no longer require one chain to discover all the modes but instead only need each mode being discovered by at least one chain. 
As a result, Equation~\ref{eq:hat-Z-k} provides a much more accurate estimator than basic methods.

An interesting point of note is that one can also use the importance samples generated by the PI-MAIS for the estimate $\hat{\pi}_k(x)$,
where $\hat \pi_k(x)$ from Equation~\ref{eq:hat-pi-k} will be
\begin{align}
\hat \pi_k(x) := \sum_{n=1}^{N}\sum_{t=1}^{T_k}\sum_{m=1}^{M} \bar w^{(k)}_{n, t, m}\delta_{\tilde x^{(k)}_{n, t, m}}(\cdot), 
\quad \text{where} \; 
\bar w^{(k)}_{n, t, m} := {\tilde w^{(k)}_{n, t, m}} \bigg/ {\sum_{n=1}^{N}\sum_{t=1}^{T_k}\sum_{m=1}^{M} \tilde w^{(k)}_{n, t, m}}.
\end{align}
The relative merits of these approaches depend on the exact problem. 
For problems where the PI-MAIS forms an efficient adaptive importance sampler, the estimate it produces will be preferable.
However, in some cases, particularly high-dimensional problems, this sampler may struggle, so that it is more effective to take the original MCMC samples.
Though it might seem that we are doomed to fail anyway in such situations, as the struggling of the PI-MAIS estimator is likely to indicate our $Z_k$ estimates are poor, this is certainly not always the case.
In particular, for many problems, one SLP will dominate, i.e. $Z_{k^*} \gg Z_{k\neq k^*}$ for some $k^*$.
In that case, we do not necessarily need accurate estimates of the $Z_k$'s to achieve an overall good approximation of the posterior. 
We just need to identify the dominant $Z_k$.

\section{Additional Details on SLP Extraction}
\label{sec:app:disc}

To better understand our SLP extraction procedure, one can imagine that we maintain two stacks of information of SLPs: one \texttt{total} stack and one \texttt{active} stack ($A^{total}$ and $A^{active}$ respectively in Algorithm~\ref{alg:div-con-anglican}).
The \texttt{total} stack records all the information of all discovered SLPs and the \texttt{active} stack keeps the SLPs that are believed to be promising so far.

Let's now have a detailed look at Algorithm~\ref{alg:div-con-anglican} together.
To prevent the rate of models being generated from outstripping our ability to perform inference on current models, we probably only want to perform inference on a subset of all possible sub-models given finite computational budget, which is $A^{active}$.
However, to determine which SLP might be good, i.e. have high posterior mass, is somewhat part of the job of the inference.
Therefore, in DCC, we propose that if a model in $\{A^{total}\setminus A^{active} \}$ is ``close'' enough to the promising models discovered so far as in $A^{active}$, it would be regarded as being \emph{potentially good} and added to $A^{active}$. 
To quantify the closeness,
we count how many times a discovered SLP not in $A^{active}$
gets proposed by a model in $A^{active}$ during the global exploration step (at line $13$, Algorithm~\ref{alg:div-con-anglican}).
We will add a newly discovered SLP into $A^{active}$ for the resource allocation only when its count reaches some threshold $C_0$ (line $3$).

One might worry the number of models in $A^{active}$ might still go beyond the capacity of the inference engine. 
We have considered the following design choices to avoid this situation in the DCC in Anglican.
The first one is to increase $C_0$ accordingly as the iteration grows. 
Intuitively, an SLP needs to be proposed more often to demonstrate that it might be a good one
when more computational resources are provided.
Another design choice is to control the total number of sub-models in $A^{active}$.
For instance, before one can add a new model in  $A^{active}$, one needs to take an SLP out of $A^{active}$, e.g. the one with the least $\hat \gamma_k$, if the upper bound of the total ``active'' number has reached. 
Our DCC also randomly chooses one ``non-active'' SLPs to perform local inference to ensure the overall correctness.

These design choices, though, are not always necessary if the number of possible sub-models does not explode naturally. 
For example, in the GMM example, this number is controlled by the value of a Poisson random variable which diminishes quickly, in which case we do not need to further bound the number of active sub-models.
But when it comes to the GMM with the misspecified prior where the possible number of active models can easily grow quickly, these design choices become essential. 
They prevent too much computation resources from being waste on keeping discovering new sub-models rather than being used to perform inference in the discovered ones. 

From the practical perspective, one might also want to ``split'' on discrete variables as their values are likely to affect the downstream program path. 
This means that for specific discrete variable(s), not only their addresses but also their values are included in defining a program path. 
It equivalently transforms sampling a discrete variable to observing the variable being a fixed value. 
The benefit of doing so is when performing inference locally, we no longer need to propose changes for that discrete variable but instead evaluate its conditional probability.
Therefore, we can ``avoid'' proposing samples out of current SLP too often to waste the computation.
The posterior of that discrete variable can be obtained from the local marginal likelihood estimates.
In our implementation of DCC in Anglican, we require the user to specify which discrete variables that they want to ``split'' on.
Automatically distinguishing which discrete variable will or will not affect program paths is beyond the scope of this paper, and we shall leave it for future work. 

\section{Additional Details on Resource Allocation}
\label{sec:supp-resource-alloc}

Once a new candidate of SLP has been selected to be added into $A_{active}$,
DCC firstly performs $T_{init}$ optimization steps to boost initialization of inference. 
Informally, we want to burn in the MCMC chains quickly 
such that the initialization of each chain would be close to the local mode of the target distribution.
By doing so, DCC applies a ``greedy'' RMH where it only accepts the MCMC samples with the larger $\hat \gamma_k$, which enforces the hill-climbing behavior in MCMC to discover modes.
Note that only the MCMC samples of the last step in the optimization will be stored as the initialization of each chain and therefore this optimization strategy will not affect the correctness of the local inference. 

As introduced in \S~\ref{sec:method-PPS-conquer-resource}, the resource allocation scheme is based on an Upper Confidence Bound~(UCB) scheme~\cite{carpentier2015adaptive} developed by~\citet{rainforth2018inference}.
We recall the utility function for each SLP being
\begin{align}
\label{eq:ucb}
U_k := \frac{1}{S_k} \left(
\frac{(1-\delta)\hat{\tau}_k}{\max_k \{\hat{\tau}_k\}} 
+ \frac{\delta \hat{p}_k}{\max_k \{\hat{p}_k\}} 
+  \frac{\beta \log \sum_{k} S_k}{\sqrt{S_k}}
\right)
\end{align}
where 
$S_k$ is the number of times that $A_k$ has been chosen to perform local inference so far, 
$\hat \tau_k$ is the ``exploitation target'' of $A_k$,
$\hat p_k $ is a target exploration term, 
and $\delta$ and $\beta$ are hyper-parameters.

As proved by~\citet[\S 5.1]{rainforth2018inference}, 
the optimal asymptotic allocation strategy is to choose each $A_k$ in proportion to $\hat \tau_k = \sqrt{ Z_k^2 + (1+\kappa)\sigma_k^2 }$
where $\kappa$ is a smoothness hyper-parameter,
$Z_k$ is the local marginal likelihood, and $\sigma_k^2$ is the variance of the weights of the individual samples used to generate $Z_k$.
Intuitively, this allocates resources not only to the SLPs with high marginal probability mass, but also to the ones having high variance on our estimate of it.
We normalize each $\hat \tau_k$ by the maximum of $\hat \tau_{1:K}$ as the reward function in UCB is usually in $[0,1]$.

The target exploration term $\hat p_k $ is a subjective tail-probability estimate on how much the local inference \textit{could improve} in estimating the local marginal likelihood if given more computations.
This is motivated by the fact that estimating $Z_k$ accurately is difficult, especially at the early stage of inference. 
One might miss substantial modes if only relying on optimism boost to undertake exploration.
As per~\cite{rainforth2018inference}, we realize this insight by extracting additional information from the log weights. 
Namely, we define $\hat p_k := P(\hat w_k(T_a) > w_{\mathit{th}} )\approx 1-\Psi_k(\log w_{\mathit{th}})^{T_a}$, 
which means the probability of obtaining at least one sample with weight $w$ that exceeds some threshold weight $w_\mathit{th}$ if provided with $T_a$ ``look-ahead'' samples. 
Here $\Psi_k(\cdot) $ is a cumulative density estimator of the log local weights (eg. the cumulative density function for the normal distribution), $T_a$ is a hyperparameter, and $w_{\mathit{th}}$ can be set to the maximum weight so far among all SLPs.
If $\hat p_k$ is high, it implies that there is a high chance that one can produce higher estimates of $Z_k$ given more budget.

\section{Details on Experiments}
\label{sec:supp-experiments}
	

\begin{figure}[!t]
	\centering
	\begin{minipage}{\textwidth}
	\begin{lstlisting}
(defdist lik-dist
    [mus std-scalar]
    []   ; auxiliary bindings
    (sample* [this] nil)  ;; not used
    (observe* [this y]  ;; customize likelihood
        (reduce log-sum-exp
            (map #(- (observe* (normal %1 std-scalar) y) (log (count mus))) 
                  mus))))

(with-primitive-procedures [lik-dist]
  (defquery gmm-open [data]
     (let [poi-rate 9
           ;; sample the number of total clusters
           K (+ 1 (sample (poisson poi-rate)))
           lo 0.
           up 20.
           ;; sample the mean for each k-th cluster
           mus (loop [k 0
                      mus []]
                 (if (= k K)
                    mus  ;; return mus
                    (let [mu-k (sample (uniform-continuous 
	                    		(+ lo (* (/ k K) (- up lo)))
	                    		(+ lo (* (/ (+ k 1) K) (- up lo)))))]
	               (recur (inc k) (conj mus mu-k)))))
            obs-std 0.1]
	 ;; evaluate the log likelihood
	 (map (fn [y] (observe (lik-dist mus obs-std) y)) data)
	 ;; output
	 (cons K mus))))
			\end{lstlisting}

	\end{minipage}
	\vspace{-12pt}
	\caption{Code example of GMM in Anglican}
	\label{fig:GMM-anglican-code}
	\vspace{-5pt}
\end{figure}

\subsection{Gaussian Mixture Model}
\label{sec:supp-gmm}
The Gaussian Mixture Model defined in \S\ref{sec:inf-eng} can be written in Anglican as in Figure~\ref{fig:GMM-anglican-code}. 
The kernel density estimation of the data is shown Figure~\ref{fig:GMM-demo} (black line) with the raw data file provided in the code folder.

\subsection{Function Induction}
\label{sec:supp-pcfg}
The Anglican program for the function induction model in in \S\ref{sec:pcfg} is shown in Figure~\ref{fig:PCFG-anglican-code}.
The prior distribution for applying each rule $R = \{e \to x \,\mid\, x^2 \,\mid\, \sin(a*e) \,\mid\, a*e + b*e \}$ is $P_R = [0.3, 0.3, 0.2, 0.2]$. 
To control the exploding of the number of sub-models, we set the maximum depth of the function structure being three and prohibit consecutive plus.  
Both our training data (blue points) and test data (orange points) displayed in Figure~\ref{fig:PCFG-qualitative} are generated from $f(x) = -x + 2\sin (5x^2)$ with observation noise being $0.5$ and the raw data files are provided in the code folder.


\begin{figure}[!t]
	\centering
	\begin{minipage}{\textwidth}
	\begin{lstlisting}
(defm gen-prog [curr-depth max-depth prev-type]
    (let [expr-type (if (< curr-depth max-depth) 
                       (if (= prev-type 3)
                          (sample (discrete [0.35  0.35  0.3]))
                          (sample (discrete [0.3  0.3  0.2 0.2])))     
                       (sample (discrete [0.5 0.5])))]
        (cond
            (= expr-type 0) 
            (if (nil? prev-type)
                (let [_ (sample (normal 0 1))]
                   'x)
                'x)
            
            (= expr-type 1) 
            (let [_ (sample (normal 0 1))]
               (list '* 'x 'x))
            
            (= expr-type 2) 
            (let [a (sample (normal 0 1))
                  curr-depth (+ curr-depth 1)
                  expr-sin (list 'Math/sin 
                         (list '* a (gen-prog curr-depth max-depth expr-type)))]
               expr-sin)
            
            (= expr-type 3) 
            (let [a (sample (normal 0 1))
                  b (sample (normal 0 1))
                  curr-depth (+ curr-depth 1)
                  expr-plus (list '+ 
                         (list '* a (gen-prog curr-depth max-depth expr-type))
                         (list '* b (gen-prog curr-depth max-depth expr-type)))]
               expr-plus))))

(defm gen-prog-fn [max-depth]
    (list 'fn ['x] (gen-prog 1 max-depth nil)))

(defquery pcfg-fn-new [ins outs]
    (let [obs-std 0.5
          max-depth 3
          f (gen-prog-fn max-depth)
          f-x (mapv (eval f) ins)]
        (map #(observe (normal %1 obs-std) %2) f-x outs)
        f))
			\end{lstlisting}
	\end{minipage}
	\vspace{-12pt}
	\caption{Code example of the Function Induction model in Anglican}
	\label{fig:PCFG-anglican-code}
	\vspace{-5pt}
\end{figure}

\end{document}